\documentclass[10pt,journal,compsoc]{IEEEtran}
\usepackage{amsmath,amsfonts}
\usepackage{algorithmic}
\usepackage{algorithm}
\usepackage{array}
\usepackage{subcaption}
\usepackage{textcomp}
\usepackage{stfloats}
\usepackage{url}
\usepackage{verbatim}
\usepackage{graphicx}
\usepackage{mathtools}

\usepackage{amsthm}

\theoremstyle{definition}
\newtheorem{definition}{Definition}[]
\newtheorem{rem}{Remark}[]
\newtheorem{prop}{Proposition}
\usepackage{makecell}
\usepackage{multirow}
\usepackage{hyperref} 
\usepackage{xcolor}   
\hypersetup{
    colorlinks=true,
    linkcolor=orange,
    filecolor=green, 
    urlcolor=black,
    citecolor=blue
}

\ifCLASSOPTIONcompsoc
  \usepackage[nocompress]{cite}
\else
  \usepackage{cite}
\fi
\ifCLASSINFOpdf
\else
\fi
\begin{document}
%
\title{Skeletonization Quality Evaluation: Geometric\\Metrics for Point Cloud Analysis in Robotics}
%
%
%
%

\author{Qingmeng Wen,~\IEEEmembership{Student Member,~IEEE}, Yu-Kun Lai, ~\IEEEmembership{Senior Member,~IEEE}, Ze Ji,~\IEEEmembership{Member,~IEEE}, \\ Seyed Amir Tafrishi,~\IEEEmembership{Member,~IEEE}
\IEEEcompsocitemizethanks{\IEEEcompsocthanksitem Qingmeng Wen and Seyed Amir Tafrishi are with Geometric Mechanics and Mechatronics in Robotics (gm$^2$R) Lab, School of Engineering, Cardiff University, Queen's Buildings, The Parade, Cardiff, CF24 3AA, UK \protect\\
E-mail:  \{wen1, tafrishisa\}@cardiff.ac.uk
\IEEEcompsocthanksitem Ze Ji is with the School of Engineering, Cardiff University, CF24 3AA, UK \protect\\
E-mail:  jiz1@cardiff.ac.uk
\IEEEcompsocthanksitem Yu-Kun Lai is with the School of Computer Science and Informatics, Cardiff University, CF24 4AG, UK.\protect\\
E-mail: laiy4@cardiff.ac.uk}
\thanks{Manuscript received April 19, 2005; revised August 26, 2015.}}

%
%

\markboth{Journal of \LaTeX\ Class Files,~Vol.~14, No.~8, August~2015}%
{Shell \MakeLowercase{\textit{et al.}}: Bare Demo of IEEEtran.cls for Computer Society Journals}
%



\IEEEtitleabstractindextext{%
\begin{abstract}
Skeletonization is a powerful tool for shape analysis, rooted in the inherent instinct to understand an object's morphology. It has found applications across various domains, including robotics. Although skeletonization algorithms have been studied in recent years, their performance is rarely quantified with detailed numerical evaluations. This work focuses on defining and quantifying geometric properties to systematically score the skeletonization results of point cloud shapes across multiple aspects, including topological similarity, boundedness, centeredness, and smoothness. We introduce these representative metric definitions along with a numerical scoring framework to analyze skeletonization outcomes concerning point cloud data for different scenarios, from object manipulation to mobile robot navigation. Additionally, we provide an open-source tool to enable the research community to evaluate and refine their skeleton models. Finally, we assess the performance and sensitivity of the proposed geometric evaluation methods from various robotic applications.
\end{abstract}

}

\maketitle

\IEEEdisplaynontitleabstractindextext

%
\IEEEpeerreviewmaketitle

\section{Introduction}
\IEEEPARstart{S}{keletonization} is an abstract shape representation that captures the intuitive topological structure of a shape. Studies by Ayzenberg et al.~\cite{ayzenberg_skeletal_2019} suggest that skeletal descriptions align with human intuition for distinguishing different shapes. Driven by its promising applications, skeletal representation has been widely explored in both robotics and computer vision. However, despite numerous methods for extracting curve/surface skeletons from surfaces or point clouds, quantitative evaluations of skeletonization remain open problems~\cite{tagliasacchi20163d,saha2016survey,meyer2023cherrypicker}. In other words, determining an optimal reference skeleton model for highly complex shapes is challenging, especially beyond easily distinguishable biological forms that naturally contain skeletal structures.

One typical skeletal representation is the surface skeleton. It is also mentioned as the medial axis surface or medial axis of a 3D shape~\cite{sobiecki2014comparison}. The notion of medial axis transform is first mentioned by Blum~\cite{blum_transformation_1967} to describe a shape. For a 3D shape, the definition given by Amenta et al.~\cite{amenta2001power} suggests that the medial axis of a shape is the center collection of medial balls, where a medial ball is a maximal inscribed ball of the shape boundary. The definition of the medial axis is simple and can be applied to many fields, such as shape manipulation and surface reconstruction. However, it is argued that the medial axis is both computationally expensive and is not robust to noises yet~\cite{cornea_curve-skeleton_2007,lin_point2skeleton_2021}. Besides, most studies of the medial axis computation are designed to work on watertight input~\cite{lin_point2skeleton_2021}, while in reality, especially in the robotic scenario, the most accessible data is a noisy point cloud.

Another common skeletal representation is the curve skeleton, which has a 1D structure and is more abstract than the surface skeleton~\cite{qin_mass-driven_2020}. In comparison to the medial axis, the curve skeleton is more compact and easier to manipulate. Consequently, the curve skeleton has garnered greater interest for applications than the medial axis~\cite{cao_point_2010}. Furthermore, numerous curve skeletonization methods exist for both watertight surfaces and noisy point cloud inputs. However, a universally accepted definition of the curve skeleton remains absent~\cite{cornea_curve-skeleton_2007}. For instance, Dey and Sun~\cite{dey2006defining} define the curve skeleton as a subset of the medial surface, whereas Cornea et al.~\cite{cornea_curve-skeleton_2007} argue that this definition is overly restrictive. Another perspective considers the curve skeleton as a subset derived from Laplacian contraction applied to either a mesh or point cloud input~\cite{tagliasacchi_curve_2009,cao_point_2010,GLSkeleton2024}. However, due to its highly abstract nature, a curve skeleton may alter or omit critical geometric properties of the original object, necessitating a careful evaluation to ensure that it faithfully represents the point cloud shape.  

Skeletonization methods encode valuable topological information, making them applicable across various fields, including robotics. Cornea et al.~\cite{cornea_curve-skeleton_2007} first explored the applications of curve skeletons, emphasizing their role in virtual navigation, where the centeredness property helps avoid collisions. In robotics, topology-preserving skeletons facilitate navigation by constructing free-space bubble graphs~\cite{noel2023skeleton}. Recent studies further demonstrate their potential in cave exploration, where the environment is typically dark and uncertain~\cite{tabib2021autonomous,yang2024cross}. Curve skeletons also aid dense surface reconstruction from SLAM data~\cite{wu2020skeleton} and play a crucial role in robotic grasping. Their topological structure enables grasping of objects with holes, as shown by Pokorny et al. and Stork et al.~\cite{pokorny2013grasping, stork2013topology}, while Varava et al.~\cite{varava2016caging} leveraged skeletons to detect key features such as ``forks'' and ``necks'' for caging grasps. Both medial axis transformations and curve skeletons have been used in grasp planning~\cite{Przybylski2011, przybylski2012skeleton, vahrenkamp2018planning}. Additionally, skeletonization shows potential for applications in complex rolling contact systems~\cite{tafrishi2025survey}. Besides, skeleton information has proven beneficial in agricultural robotics, particularly in agricultural manipulation~\cite{you2022semantics, schneider20233d, kim2023occlusion}. Despite these advancements, the absence of a well-defined criterion for desirable skeleton structures in robotics remains a significant challenge.  
In previous studies, skeletonization has predominantly relied on visual judgment, with quantitative assessment remaining uncommon~\cite{tagliasacchi20163d}. The first justifiable definition of desirable properties for curve skeletons was introduced by Cornea et al.~\cite{cornea_curve-skeleton_2007}. Building on this definition, the performance of multiple existing skeletonization methods has been analyzed~\cite{sobiecki2013qualitative}, followed by a more extensive study of both surface and curve skeleton results~\cite{sobiecki2014comparison}. However, their discussion on skeleton quality still relies on subjective visual assessment. A comprehensive survey by Tagliasacchi et al.~\cite{tagliasacchi20163d} and Saha et al.~\cite{saha2016survey} reiterated the similar definition of desired properties and highlighted the need for a quantitative analysis of skeletonization methods. Recently, Laplacian-based skeletonization results have been evaluated based on surface normal vectors and curvature changes during the contraction process~\cite{wen2024criteria}. You et al. proposed an objective function to assess curve skeletonization by considering edge length, straightness, and orientation~\cite{you2022semantics}. Additionally, some researchers have compared their skeletonization results against known ground truth skeletons of input data~\cite{meyer2023cherrypicker,dobbs2024quantifying}. While existing studies provide insights into comparing skeletonization methods, quantitative and generalizable geometric metrics remain scarce, with current evaluations still heavily dependent on subjective visual judgment. Consequently, the quantitative comparison of skeletonization methods remains an open challenge due to the lack of a unified skeleton definition and the inherent difficulty in measuring skeleton properties.  

Inspired by the motivations mentioned earlier, we develop numeric and representative metrics to assess the quality of curve and surface skeletonization for point cloud shapes through explicit geometric definitions of skeleton properties. Thus, the contributions of this paper are as follows:  
\begin{itemize}
    \item We introduce formal metric and qualitative geometric evaluation definitions—topological similarity, boundedness, centeredness, and smoothness—to assess skeletonization in Section \ref{metricevaluation}.  
    \item We develop numerical and representative evaluation metrics for skeletonization, enabling comprehensive analysis of point cloud surfaces and their skeletons, with background discussion for each metric where applicable.
    \item We analyze the performance and sensitivity of the proposed evaluation properties in skeletonization, demonstrating their effectiveness using, for example, a Laplacian-based curve skeletonization method on a real-scanned dataset in Section \ref{resultsandiscussion}.  
    \item We evaluate various point cloud datasets for applications ranging from object manipulation to navigation and provide an open-access repository\footnote{\url{https://github.com/weiqimeng1/PointCloud_Skeletonization_Metrics}} for the research and development community.  
\end{itemize}



\section{skeletonization Evaluation Metrics}
\label{metricevaluation}
\subsection{Topological Similarity}
This section discusses a quantitative analysis of topology preservation in the skeletonization process. While the distance between point cloud data can be measured using various metrics, such as Hausdorff or Chamfer distance, separating pure shape metrics from distance metrics is challenging. Here, we focus on topological similarity as a key property for skeletonization results. Two topologically similar shapes must have comparable topological structures. For evaluating point cloud skeletonization in terms of topological similarity, it is necessary to design a method that compares two topological shapes represented by point sets and provides a score indicating the extent of their topological similarity.

A key concept in topology related to shape similarity is \textit{homotopy}, which refers to a deformation process that results in shapes topologically equivalent to the original ones. The deformed shapes and their corresponding originals are called homotopic or homotopy equivalent. Generally, two topologically equivalent objects must have the same number of connected components, such as cycles, holes, tunnels, cavities, and other higher-dimensional features \cite{lieutier2003any}. These properties are known as ``topological invariants''. For example, a sphere and an ellipsoid are topologically equivalent since they both have one connected component and one cavity. However, according to existing research on point cloud skeletonization \cite{cao_point_2010, Huang2013L1, lin_point2skeleton_2021}, computing a skeletal representation that meets all these criteria is neither overly complex nor always necessary. Thus, Cornea et al.~\cite{cornea_curve-skeleton_2007} proposed a more relaxed definition for curve skeletons: the topology of the original object $ O $ is preserved in a relaxed sense if the curve skeleton $ S $ includes at least one loop for each tunnel and cavity in $ O $ and retains the same number of connected components.

Although the relaxed definition of curve-skeleton homotopy equivalence accommodates skeleton properties, it remains challenging for quantitative analysis of point cloud skeletonization. Comparing topological features, such as connected components, between different shape descriptors like curve-skeletons and point cloud data is difficult. Point sets are typically unordered, unconnected, and unevenly distributed, while 1D curve skeletons consist of well-ordered, explicitly connected vertices with their distribution controlled by sampling. As a result, existing research on skeletonization comparison still heavily relies on subjective visual judgment~\cite{sobiecki2014comparison}.

While it is challenging to compare topological features directly between the curve skeletons and the original point cloud, for those methods generating contracted or shrunk point sets as an intermediate product of curve skeletons (skeletal points), it would be easier to analyze the topological change between the skeletal point sets and the original point cloud since both are the same type of shape descriptors. Let's choose the Laplacian-based point cloud skeletonization method by Cao et al.~\cite{cao_point_2010} as an example. Since the curve skeletons are extracted by sampling skeletal vertices from those skeletal points and the connections are constructed in the meantime, the precision of topological information represented by curve skeletons 
depends on the topological features of the skeletal points. 

It is usually infeasible to directly determine the homotopy equivalence between two point cloud shapes due to the inherent gap between discrete point sets and continuous manifolds. However, an alternative approach is to analyze homology groups of the shapes using persistent homology. A study by Niyogi et al.~\cite{niyogi2008finding} uncovers the possibility of accurately recovering topological invariants from discrete point sets sampled from a manifold. 

In fact, the utilization of persistent homology for shape analysis and comparison has been studied in previous studies, known as size functions. In a study by Verri and Uras \cite{verri1996metric}, the properties of the size functions and the application on shape recognition are discussed. Two computation methods for size functions are given by Frosini \cite{frosini1992measuring}, and the deformation distance and the size functions are linked by a clear definition. For more practical research on shape comparison, please refer to \cite{verri1993use}. Recently, Bergomi et al. \cite{bergomi2019towards} argued the stability of metrics is induced by persistent homology while discussing the equivalence set in machine learning. Varava et al.~\cite{varava2016caging} presented proposed control and motion planning algorithms for caging grasping by detecting non-trivial $H_1$ homology group.

\subsubsection{Problem Statement \& Preliminaries}
An ideal skeleton preserves the topology of the original shape \cite{cornea_curve-skeleton_2007}. Evaluating the preserved topology of skeletonization requires a metric that effectively captures topological shape similarity. In this part of the work, we analyze the topological changes in shape that occur during point cloud skeletonization. We assume a point set intermediate product, which is point-wise in relation to the original point cloud, is generated while performing point cloud skeletonization.

Our problem is defined within a discrete space of a point set $X=\{\mathbf{x}_1, \dots, \mathbf{x}_m\}\subseteq\mathbb{R}^{m\times 3}
$. Within the space, a point cloud shape is a point set $P$, where $P = \{\mathbf{p}_0, \dots, \mathbf{p}_n\} \subseteq X, n\leq m$. For the purposes of comparison, we define two point cloud shapes in our problem: $P_o$, which represents the original point cloud, and $P_s$, which represents the skeletal point set with approximately the same number of points as $P_o$.

Simple geometries, such as simplices, are commonly employed to represent local geometric features of $P$. In general, a simplex is the convex hull of its vertices, as depicted by the following definition.
\begin{definition}
\label{def:simlices}
    A $k$-simplex $\sigma$ is a convex hull of its $k+1$ affinely independent vertices, denoted as $\sigma:={\rm convh}\{\mathbf{v}_0, \dots, \mathbf{v}_k\}$, where $\mathbf{v}_0,\dots,\mathbf{v}_k\in \mathbb{R}^3$ are its $k+1$ vertices.    
\end{definition}
\begin{rem}
The ``affinely independent vertices'' means that given vertices $\mathbf{v}_0, \dots,\mathbf{v}_k$, the vectors $\mathbf{v}_1-\mathbf{v}_0, \dots,\mathbf{v}_k-\mathbf{v}_0$ are linearly independent. As illustrated in Fig.~\ref{fig.simlices}, simplices in dimensions of 0, 1, 2, and 3 are typically represented as a point, a line, a triangle, and a tetrahedron respectively. The convex hull of a subset of vertices in a simplex is a face of the simplex.
\end{rem} 
\begin{figure}
\centering
    \includegraphics[width=0.4\textwidth]{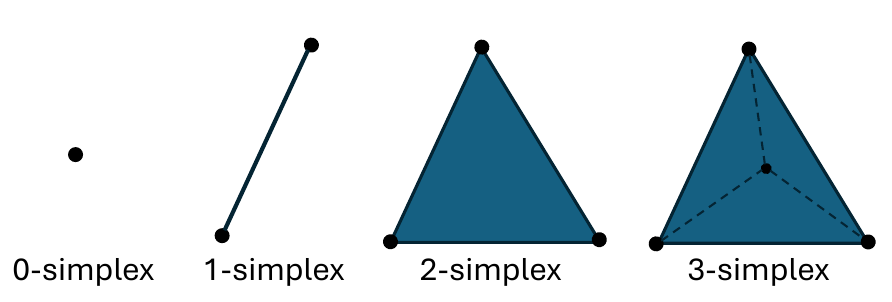}
    \caption{Some of the simplex examples.}
    \label{fig.simlices}
\end{figure}

For a more complex structure, a simplicial complex is defined by a union set of simplices, which we use to bring a defined set of point-clouds (including skeletonization) as follows
\begin{definition}
\label{def:simplicial_complex}
     A simplicial complex $K$ is a finite collection of simplices, such that if $\sigma,\sigma\prime \in K$ then $\sigma \cap \sigma\prime$ is either empty or a face of both $\sigma$ and $\sigma\prime$. 
     
     An abstract simplicial complex $\hat{K}$ is a collection of sets of geometric elements taken from the simplicial complexes $K$.
\end{definition}

\begin{rem}
    In an abstract simplicial complex, a non-zero subset of a set in the abstract simplicial complex is also contained in the abstract simplicial complex. In fact, some geometrical properties are eliminated by the abstract simplicial complex relative to the simplicial complex, and only the combination relationships are preserved.
\end{rem}

The Vietoris-Rips complex provides a natural method for constructing an abstract simplicial complex from a finite metric space and can be used to extract topological features through complex filtration~\cite{attali2011vietoris}. The given finite metric space serves as a guide for combinations of geometrical elements.
\begin{definition}
    A finite metric space $(X,\partial_X)$ of the discrete space of the point set $X$ is a metric space such that the distance between a pair of points $\mathbf{x}_i$ and $\mathbf{x}_j$ ( $\mathbf{x}_i,\mathbf{x}_j\in X$) is given by $\partial_X(\mathbf{x}_i,\mathbf{x}_j)$.
\end{definition}
\begin{definition}
\label{def.vierips}
    Given a finite metric space $(X, d_X)$ and a fixed radius $\epsilon$, the Vietoris-Rips complex ${\rm VR}_{\epsilon,P}(X, d_X)$ of point cloud shape $P$ is an abstract simplicial complex where the vertices are the points in $P$, and each $k$-simplex $\sigma = {\rm convh}\{\mathbf{v}_0, \dots, \mathbf{v}_k\} \in {\rm VR}_{\epsilon,P}(X, d_X)$ satisfies $d_X(\mathbf{v}_i, \mathbf{v}_j) < \epsilon$ for all $1 \leq i < j \leq k$.
\end{definition}
Homology in a topological space is characterized by its homology groups, defined using the boundary homomorphism~\cite{beksi20163d}. It serves as a topological invariant, allowing the comparison of point cloud shapes based on their topological properties.
\begin{figure}
\centering
    \includegraphics[width=0.2\textwidth]{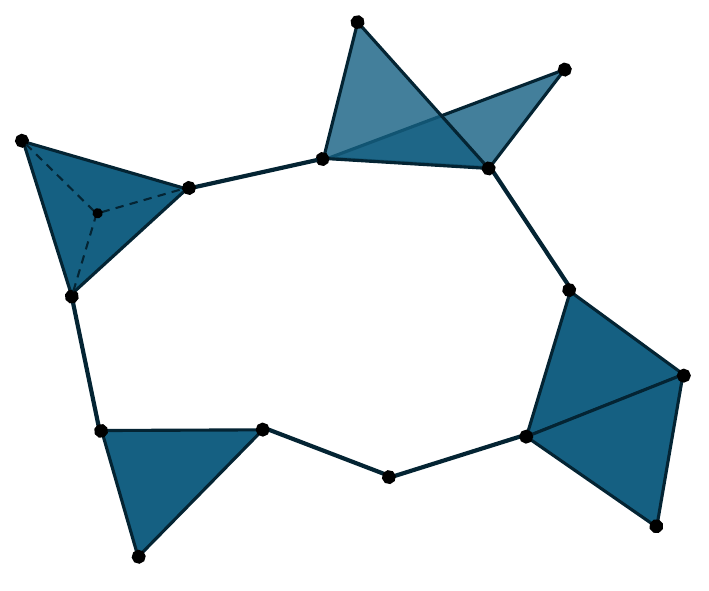}
    \caption{An example of a Vietoris-Rips complex\cite{beksi20163d}.}
    \label{fig.RipsComplex}
\end{figure}
\begin{definition}
    A boundary operator/homomorphism $\partial_d$ is the homomorphism that connects the chain complex in dimension $d$ and $d-1$, written as $\partial_d: C_d(K)\rightarrow C_{d-1}(K)$. For $i<0$, $C_i\equiv 0$.
\end{definition}

\begin{definition}
 \label{def.homology}
    Let $d$-th boundaries and cycles of $P$ built in point space be denoted as $B_d(P)={\rm im}(\partial_{d+1})$ and $Z_d(P)={\rm ker}(\partial_{d})$ respectively. The $d$-th homology group of topological space relative to point cloud $P$ is defined by
    \begin{equation}
        H_d(K):= {\rm ker}\ \partial_d/{\rm im} \ \partial_{d+1}=Z_d(P)/B_d(P).
    \end{equation}
\end{definition}
\begin{rem}
     The $d$-th homology group $H_d(K)$ of a point cloud shape represents the number of $d$-dimensional holes in the corresponding simplicial complex, describing its topological structure. 
\end{rem}

Definitions \ref{def.vierips}-\ref{def.homology} provide the mathematical foundation for the homology group of point cloud shapes. As a topological invariant, the homology group enables the abstraction and comparison of topological properties between shapes. For point clouds, this can be achieved using persistent homology.
 
\subsubsection{Methodology}
Persistent homology is a powerful tool for analyzing topological changes in point clouds. Traditional skeletonization methods define a ``good'' skeleton as one that preserves topology by remaining homotopic to the original shape~\cite{cornea_curve-skeleton_2007,sobiecki2014comparison}, but this is challenging to quantify. Inspired by Edelsbrunner et al.~\cite{edelsbrunner2002topological}, we use persistent homology to extract topological features from both the skeleton and the original point cloud, quantifying their dissimilarity through topological distances.

Let $\epsilon$ define the radius for $\epsilon$-neighborhoods used to form complexes, persistent homology tracks the evolution of topological features as $\epsilon$ increases. At $\epsilon = 0$, each point is an isolated component ($H_0$ feature). As $\epsilon$ grows, points merge into higher-dimensional simplices, creating loops ($H_1$) and cavities ($H_2$). These features appear and disappear at different rates; minor features vanish quickly, while major ones persist longer. We compute complex growth using the Vietoris–Rips complex (Definition~\ref{def.vierips}). The persistence of topological features is encoded by their birth and death times, which can be visualized using persistence barcodes~\cite{carlsson2004persistence} and compared in barcode space. As shown in Fig.~\ref{fig.barcode}, the start and end points of each bar in the barcode represent the birth and the death of a $H_0$ feature respectively.

In barcode space, the bottleneck distance and Wasserstein distance are standard metrics for measuring the dissimilarity between two barcodes. For point cloud skeletonization, the two barcodes are assumed to be generated from the original point cloud and its pointwise skeletal representation. Let $[a_1, b_1)$ and $[a_2, b_2)$ be two persistence intervals, the $\infty$-distance between the two intervals is defined as:  
\[
    d_\infty([a_1, b_1], [a_2, b_2]) = \max(|a_1 - a_2|, |b_1 - b_2|).
\]
Since only the major topological features are typically expected to be preserved, and these features exhibit the greatest persistence, it is necessary to filter out minor local features before making comparisons. Let $ P_o $ and $ P_s $ denote the original point cloud and the skeletal point set, respectively. The maximum nearest-neighbor distance for any point $ \mathbf{p}_i \in P_o $ is given by $ \epsilon^* $, defined as
\begin{equation}
   \epsilon^*=\sup_{\mathbf{p}_{o,i}\in P_o} \Vert \mathbf{p}_i - \mathbf{\psi}(\mathbf{p}_i) \Vert,
   \label{eq.epsilon_th}
\end{equation}
where $ \psi $ is the mapping between a point $ \mathbf{p}_i $ and its nearest neighbor in $ P_o $. We assume that local minor features, such as small connected components, disappear once all points at least have one connection to their neighbor points. For instance, a point that forms the smallest connected component is born at the beginning ($ \epsilon = 0 $) and disappears when it is connected to its nearest neighbor. Therefore, bars in the persistence barcode representing minor features can be eliminated by removing those bars with persistence less than $ \epsilon^* $. Let $ B_o $ and $ B_s $ be the filtered barcodes of $ P_o $ and $ P_s $, respectively. Note that the sizes of $ P_o $ and $ P_s $ are normalized to fit within a standard cubic bounding box for better comparison. The \textit{bottleneck distance} between the filtered barcodes is then given by:
\begin{equation}
    d_B(B_o, B_i) = \inf_{\phi} \sup_{Z \in B_o} d_\infty(Z, \phi(Z)),
    \label{eq.bottleneck_dis}
\end{equation}
where $\phi$ ranges over all possible bijections between $B_1$ and $B_2$. And the normalized \textit{p-Wasserstein distance} is defined as:
\begin{equation}
    d_{W_p}(B_1, B_2) = \frac{1}{n_{b}}\left( \inf_{\phi} \sum_{Z \in B_1} d_\infty(Z, \phi(Z))^p \right)^{\frac{1}{p}},
    \label{eq.wasserstein_dis}
\end{equation}
where the distance is normalized by $n_{b}$, the number of selected bars.

\begin{prop}
\label{prop.topological_smilarity}
The topological similarity between two normalized point clouds, $P_o$ and $P_s$, within a bounding box of diagonal $\epsilon_{\text{max}}$, is determined by the distance between their most persistent homology features, calculated from the growth of the Vietoris–Rips complex (Definition~\ref{def.vierips}). The persistence patterns of homological features are represented in barcode space and filtered by a threshold $ \epsilon^* $ given in Eq.~(\ref{eq.epsilon_th}). We assume $ B_o $ and $ B_s $ are the persistent barcodes of $ P_o $ and $ P_s $, where the start and end points of each bar correspond to the birth and death values of an $ H_0 $ feature. Also, the dissimilarity $ d_{o,s} $ between $ B_o $ and $ B_s $ is measured using appropriate distance metrics in barcode space, such as the bottleneck distance given in Eq.~(\ref{eq.bottleneck_dis}) or the Wasserstein distance in Eq.~(\ref{eq.wasserstein_dis})\cite{carlsson2004persistence}.

Now, let $ N_o $ denote the number of points in $ P_o $ and $ \epsilon_{\text{max}} $ the maximum value of $ \epsilon $ (Definition~\ref{def.vierips}) in point cloud space. The topological similarity between $ P_o $ and $ P_s $ is quantified as:
\[
\begin{cases} 
\text{High Similarity}, & \text{if } d_{o,s} < d^* \text{ where } 0 < d^* \leq \epsilon_{\text{max}}, \\
\text{Low Similarity}, & \text{if } d_{o,s} \geq d^*.
\end{cases}
\]
As $ d^* \to 0 $ and $ N_o \to \infty $, this comparison provides an accurate measure of topological similarity.
\end{prop}

It has been proved by Niyogi et al.~\cite{niyogi2008finding} that it is able to recover the homology of shapes from sufficient discrete samples. Homology group, as a topological invariant, is able to be applied for shape analysis \cite{verri1996metric,verri1993use,frosini1992measuring}. 

\begin{proof}
Let $P_o$ be a discrete point set consisting of $N_o$ points evenly sampled from the surface $S_o$ of shape $o$. The largest nearest-neighbor distance among all points $\mathbf{p}_i \in P_o$ is denoted as $\epsilon^*$ (Eq.~\ref{eq.epsilon_th}). The surface $S_o$ can be reconstructed by connecting each point $\mathbf{p}_i \in P_o$ to all its $\epsilon$-neighbors, assuming that every point in $P_o$ is a valid sample from $S_o$. Let $S_r$ denote the reconstructed surface.

As $N_o \to \infty$, it follows that $\epsilon^* \to 0$ and $S_r \to S_o$. This implies that, with sufficiently dense sampling, the reconstructed surface approximates the original surface with increasing accuracy. Consequently, the homology features of the sampled point set will converge to those of the original shape. Furthermore, as $\epsilon^* \to 0$, the persistence of local $H_0$ features, which primarily represent noise, diminishes rapidly. Since $\epsilon^*$ is the largest nearest-neighbor distance in the point cloud, it primarily affects local connectivity rather than global structure. When $\epsilon^*$ is sufficiently small, computing the $\epsilon$-Vietoris–Rips complex with a small $\epsilon$ already ignores these spurious $H_0$ features. Meanwhile, major topological features, such as global connectivity patterns, remain unaffected. This behavior is illustrated in Figs.~\ref{fig.SimpleGeo_barcode} (c), (i), (d), and (j).
\end{proof}

\begin{rem}
The analysis of simple geometries uncovers a fundamental link between topological changes and the persistence barcodes of homology. As shown in Fig.~\ref{fig.SimpleGeo_barcode}, deformations that alter connectivity (e.g., Figs.~\ref{fig.SimpleGeo_barcode} (c), (g)) usually result in significant changes in the longest segments of the barcode. Conversely, deformations that preserve topology (e.g., Fig.~\ref{fig.SimpleGeo_barcode} (a), (e)) maintain consistent patterns in the longest bars. These effects are quantitatively captured by the corresponding values of Wasserstein and bottleneck distances. Notably, the bottleneck distance shows greater robustness to non-topological changes, highlighting its stability as a metric for assessing topological differences.

\begin{figure}[t!]
\centering
\subfloat[Similar topology]{\includegraphics[width=0.23\textwidth]{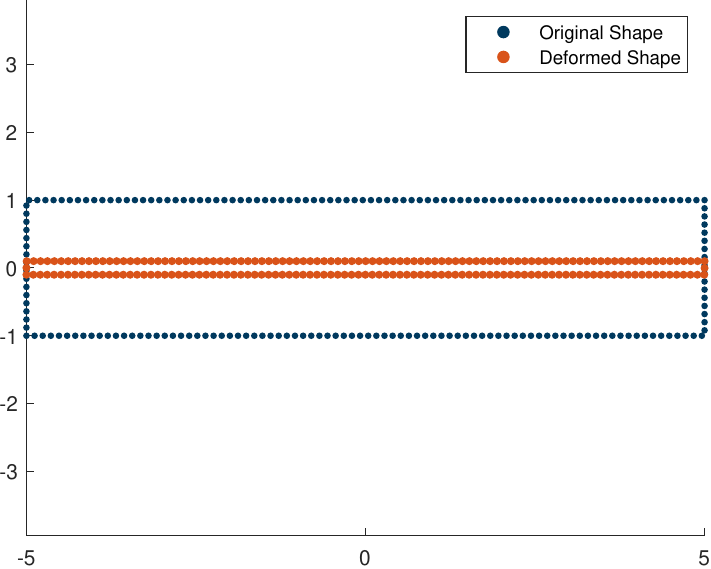}}
\subfloat[$d_B=0.018,\ d_{W_p}=0.160$]{\includegraphics[width=0.23\textwidth]{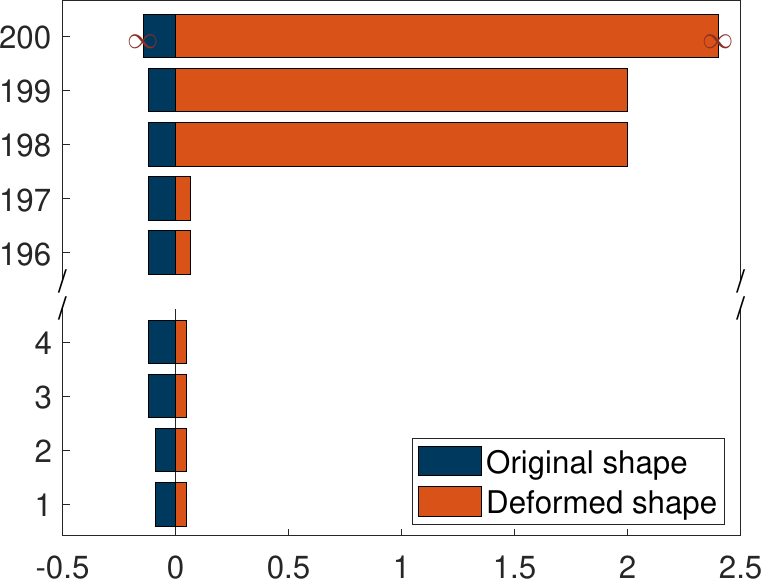}}\\
\subfloat[Changed topology]{\includegraphics[width=0.23\textwidth]{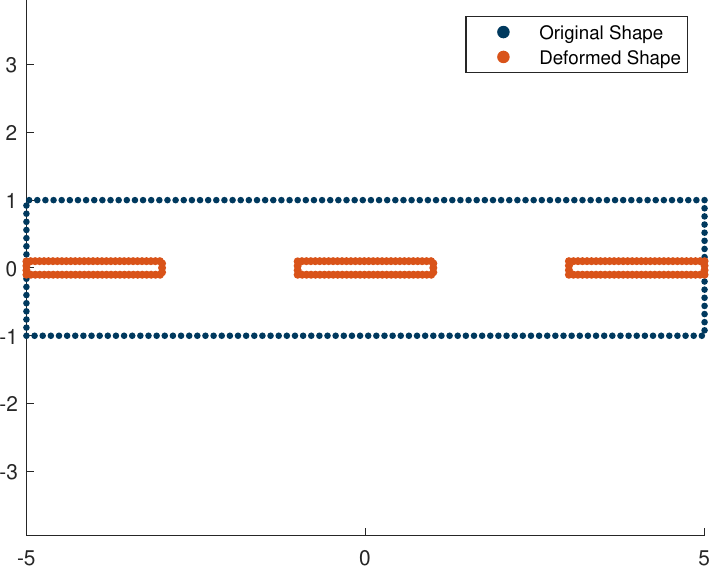}}
\subfloat[$d_B=1.007,\ d_{W_p}=0.232$]{\includegraphics[width=0.23\textwidth]{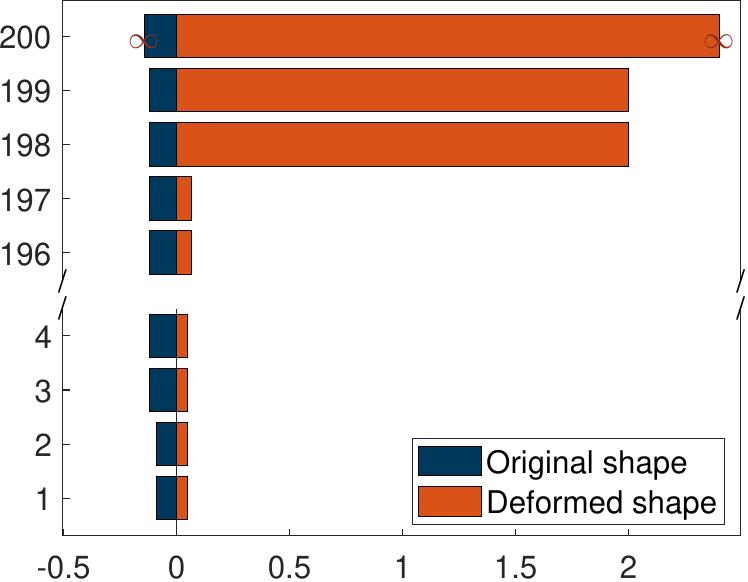}}\\
\subfloat[Similar topology]{\includegraphics[width=0.23\textwidth]{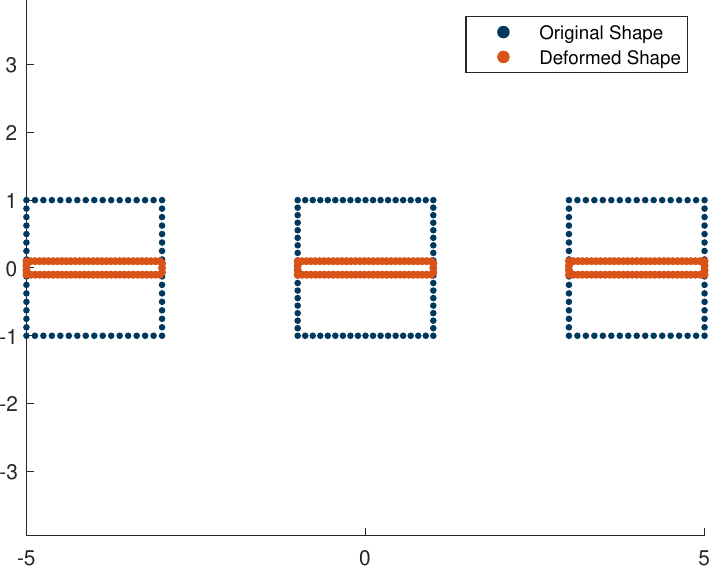}}
\subfloat[$d_B=0.057,\ d_{W_p}=0.043$]{\includegraphics[width=0.23\textwidth]{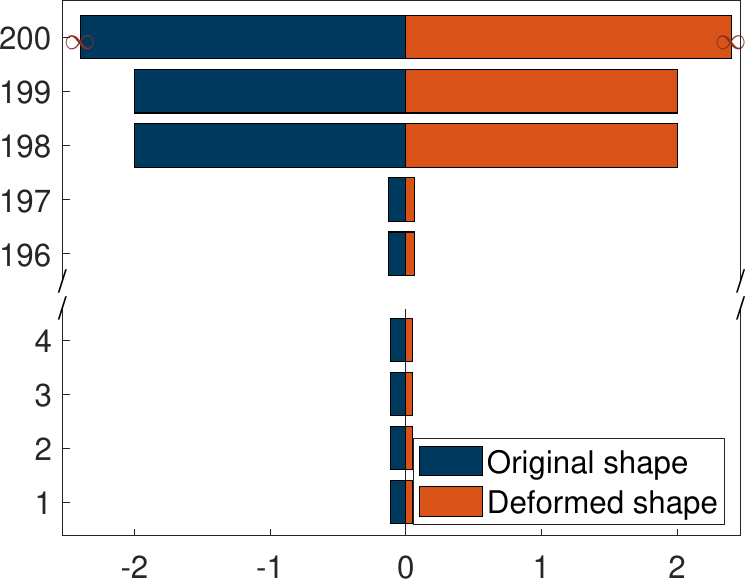}}\\
\subfloat[Changed topology]{\includegraphics[width=0.23\textwidth]{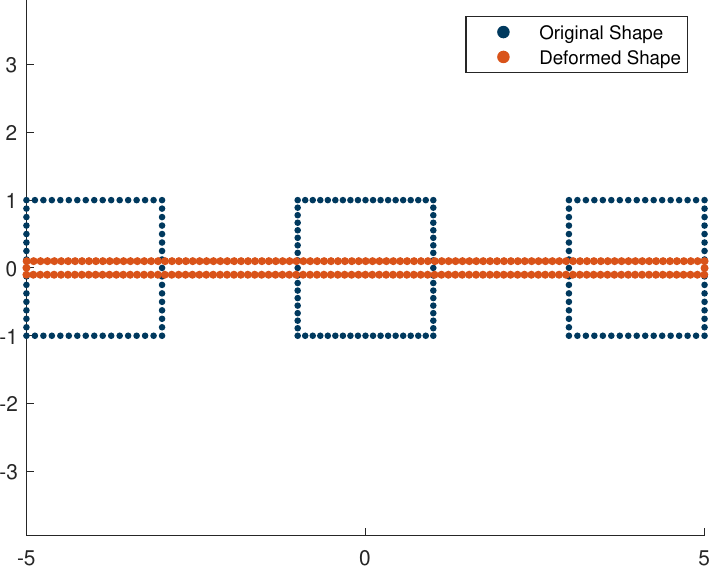}}
\subfloat[$d_B=1.000,\ d_{W_p}=0.208$]{\includegraphics[width=0.23\textwidth]{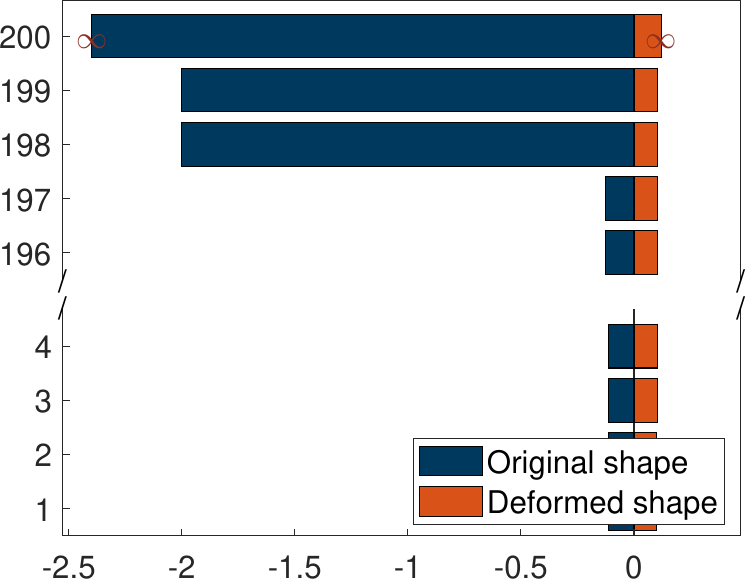}}\\
\subfloat[Dense Points]{\includegraphics[width=0.23\textwidth]{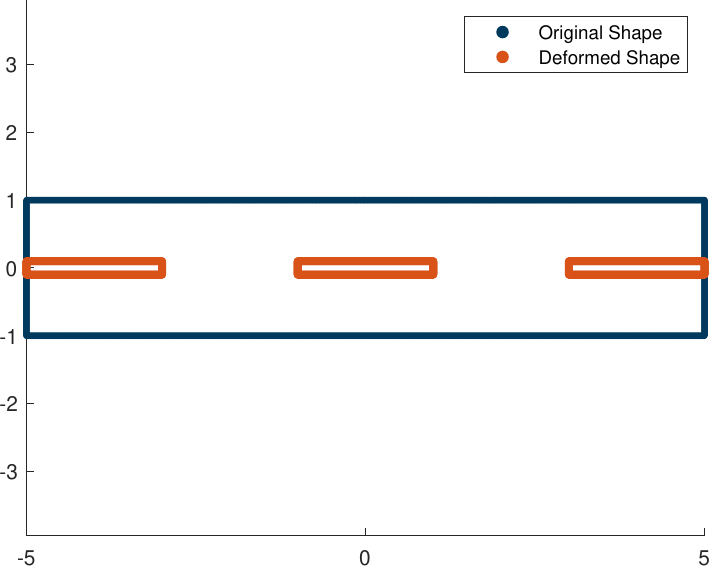}}
\subfloat[$d_B=1.006,\ d_{W_p}=0.025$]{\includegraphics[width=0.23\textwidth]{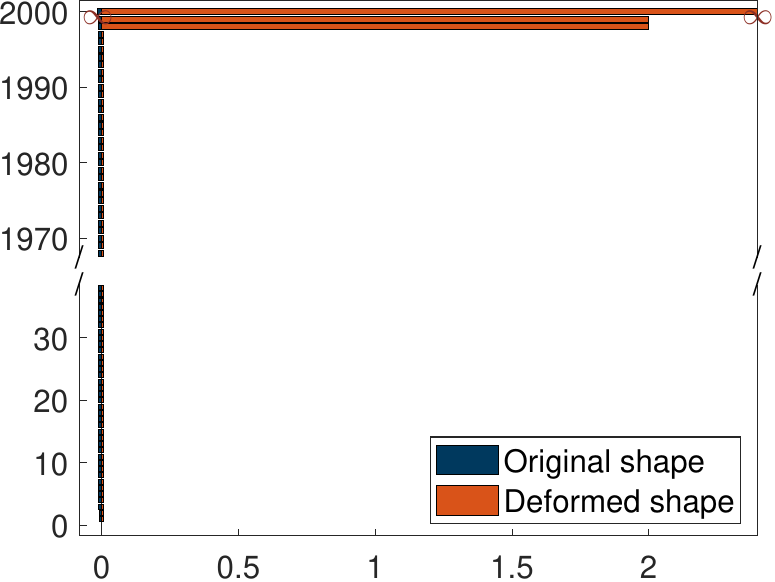}}
\caption{Persistent Homology Analysis of Simple Geometries ($H_0$ features). Shape of (a), (c), (e), (g) are with 200 points, while shape of (i) are with 2000 points.}
\label{fig.SimpleGeo_barcode}
\end{figure}
\end{rem}

\begin{rem}
    This topological similarity evaluation method is only applicable to the point-wise deformed point cloud shapes. And, it would be sensitive to the density of points since the most dominated bars are filtered with the largest distance value of the points in the original point cloud shape to their nearest neighbors.
\end{rem}

Fig.~\ref{fig.barcode} shows the persistence patterns of homology features for the original and skeletal point set data of the hammer and biscuit shapes. In Fig.~\ref{fig.barcode}(a), the skeletal points are pushed inward, reducing their distances to neighbors. This results in faster merging of connected components and shorter persistence for $H_0$ features, with local connections being formed more quickly while global features remain stable. The points move primarily in a radial direction, reducing distances more in this direction than along the axial direction, which corresponds to the medial axis. Good skeletonization preserves topological features along the medial axis. In the hammer shape (Fig.~\ref{fig.barcode}(a)), longer-lasting homology features indicate less change than in the biscuit shape (Fig.~\ref{fig.barcode}(b)), where the topology is significantly altered, including inconsistent spikes. The property of topology preservation in the skeletonization process is quantified by the barcode pattern distances. The contraction of the hammer shape results in smaller bottleneck and Wasserstein distances for the homology persistence barcode compared to the biscuit shape, as shown in the caption of subfigures of Fig.~\ref{fig.barcode}.

In summary, persistent homology quantifies topological similarity by comparing the persistence of homology features using an appropriate distance metric, such as bottleneck or Wasserstein distances, between the barcode patterns of the original and skeletal point sets.

\begin{figure}[t!]
\centering
    \subfloat[$d_B=0.0084$,\\$d_{W_p}=0.0076$]{\includegraphics[width=0.23\textwidth]{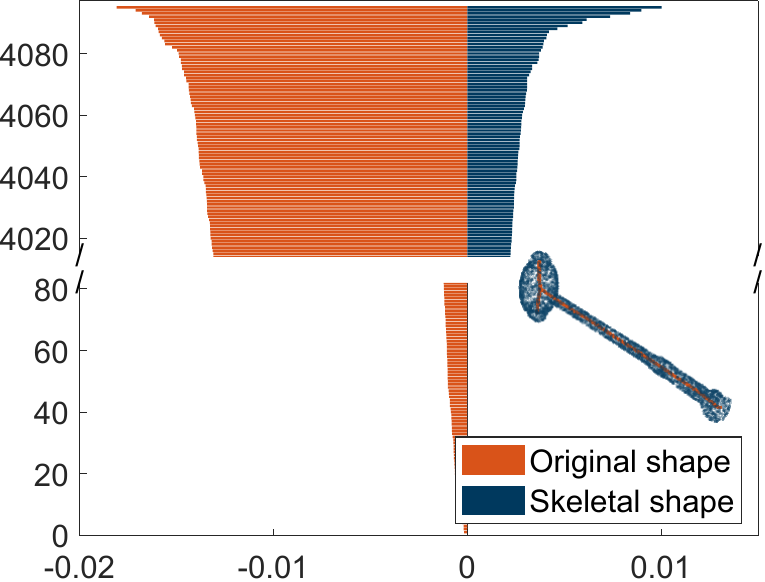}\label{fig.barcodeA}}
    \subfloat[ $d_B=0.0561$, \\$ d_{W_p}=0.0203$]{\includegraphics[width=0.23\textwidth]{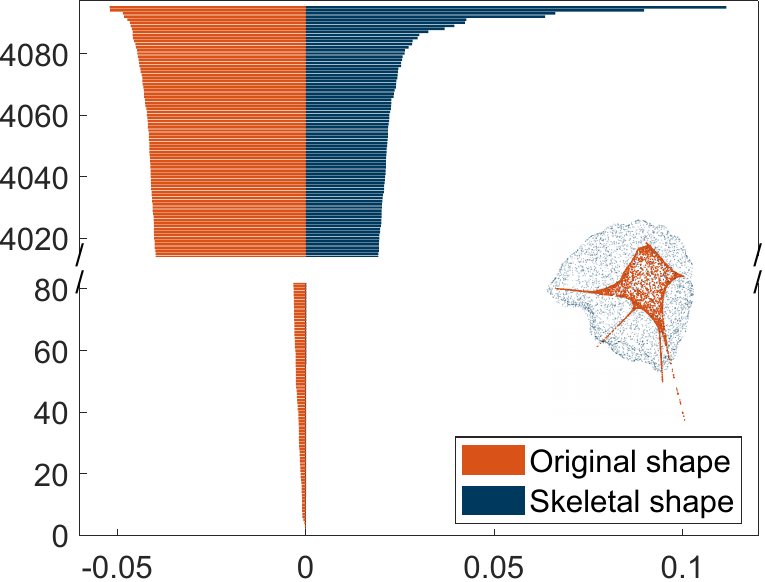}\label{fig.barcodeB}}
    \caption{The barcode of persistent homology ($H_0$ features). The input point clouds are scaled to fit within a cubic bounding box whose diagonal is 1.6 and only the top 5\% and bottom 5\% of persistence bars are shown for clarity)}
    \label{fig.barcode}
\end{figure}
 
\subsection{Boundedness}
The skeleton of a shape represented by a point cloud is expected to lie within the shape, or, in other words, to be bounded by the original shape surface~\cite{sobiecki2013qualitative, cornea_curve-skeleton_2007}, before another characteristic, centeredness, can be discussed. However, the challenge problem is that the surface represented by point cloud data is implicit and unclear. This subsection will discuss the boundedness of the skeletal points/vertices. Given a point in the point cloud space, our method is able to give a score that depicts the bounding state of the point relative to the surface of a shape. Based on the bounding state of the point, the overall bounding state of a resultant skeletal shape is also given. 

\subsubsection{Problem statement}
As shown by the bottom-right skeletonization result of Fig.~\ref{fig.barcode}(b), skeletal points or vertices extending beyond the boundary of the shape surface are generally undesirable and should be considered incorrect convergence/transformation. \textit{Boundedness} refers to the property of the resultant skeletons whereby all skeletal elements are expected to remain within the original shape. For point cloud shapes, this presents a significant challenge, as the surface representation by discrete points is inherently ambiguous.

In this subsection, boundedness is defined with respect to points. A skeletal point set $P_s \subset\mathbb{R}^{m\times3}$ consists of skeletal points $ \{ \mathbf{p}_{s,i} \mid \mathbf{p}_{s,i} \in \mathbb{R}^3 \} $, corresponding to the original point cloud $P_o\subset\mathbb{R}^{n\times3}$. The shape represented by the point cloud $P_o$ is denoted as $\Omega$. Each point $\mathbf{p}_{s,i}\in P_s$ is expected to be bounded by the shape boundary $\partial\Omega$. 

A curve skeleton $G_s$ corresponding to the shape $\Omega$ is represented as a graph structure $(V, E)$, where $V \subseteq \mathbb{R}^3$ is the set of skeletal vertices, and $E \subseteq V \times V$ is the set of skeletal edges. Each element $e_i \in E$ or $\mathbf{v}_i \in V$ is expected to be bounded by $\partial\Omega$.

\subsubsection{Methodology}
Let $ P $ be an $ n $-point point cloud in a discrete space of point set $ X \subseteq \mathbb{R}^{m \times 3} $, and let $ \mathbf{x}_i \in X $ be an arbitrary point in $ X $. We define the direction vector from $ \mathbf{x}_i $ to $ \mathbf{p}_i $ as  
\begin{equation}
    \mathbf{d} (\mathbf{x}_i, \mathbf{p}_i) \coloneqq \frac{\mathbf{p}_i - \mathbf{x}_i}{\|\mathbf{p}_i - \mathbf{x}_i\|}.
    \label{eq.direction_vec}
\end{equation}
where $ \mathbf{p}_i \in P $ is an arbitrary point in $ P $. If we interpret the normalized direction vectors from all other points in $ P $ relative to $ \mathbf{x}_i $ as coordinates, we observe that these points are projected onto the surface of a unit sphere. This projection will result in an intact sphere surface if $ \mathbf{x}_i $ is fully bounded by a shape represented by the point cloud $ P $, as illustrated in Fig.~\ref{fig.bound_proj}. Conversely, suppose $ \mathbf{x}_i $ is only partially bounded. In that case, there will be regions on the sphere surface with no points, resulting in holes on the sphere surface in simple language. Based on this phenomenon, the following definitions are given for obtaining proper boundedness metrics:  
\begin{definition}
\label{def.boundedness_pts}
Based on the assumption that the shape represented by $P$ is a pure convex shape, the boundedness of $ \mathbf{x}_i $ by point cloud shape $P$ can be evaluated by measuring the areas on the sphere surface filled by the projected points. Let $ \mathbf{p}_{\mathbf{x}_i ,\mathbf{p}_i} = \mathbf{d} (\mathbf{x}_i, \mathbf{p}_i)$ (Eq.~\ref{eq.direction_vec}) denote the point resulting from the projection of $ \mathbf{p}_i $ relative to $ \mathbf{x}_i $, and let the corresponding set of sphere points be $ \mathcal{P}_{\mathbf{x}_i , P} $. The boundedness metric of $ \mathbf{x}_i $ relative to point cloud shape $P$ is defined as  
\begin{equation}
    \beta_{\mathbf{x}_i \circ P} \coloneq \frac{S_{\mathbf{x}_i, P}}{4 \pi r^2},
    \label{eq.boundedness_pts}
\end{equation}
where $ r = 1 $ (for a unit sphere) and $ S_{\mathbf{x}_i, P} $ is the total area covered by the points $ \mathbf{p}_{\mathbf{x}_i , \mathbf{p}_i} \in \mathcal{P}_{\mathbf{x}_i , P} $.  
\end{definition}

\begin{definition}
\label{def.curve_skeleton_point}
    Given a curve skeleton $ G_s = (V, E) $, an edge $ e_i = (\mathbf{v}_j, \mathbf{v}_k) \in E $ can be parameterized as  
    \[
        e_i(t) = t \mathbf{v}_j + (1 - t) \mathbf{v}_k, \quad t \in [0, 1].
    \]
   A point $ \mathbf{p}_{g}$ of the curve skeleton $G_s$ is given by $ \mathbf{p}_{g}  = e_i(t) $, where $ t \in [0, 1]$ and $e_i$ is an arbitrary edge of $G_s$.
\end{definition}

\begin{definition}
Let $ P_o, P_s \subseteq X $ be the skeletal points in a discrete space of $ X $. Let $ N_s $ be the total number of points in $ P_s $ and $ N_{s,b} $ be the number of points in $ P_s $ bounded by $ P_o $ (Proposition~\ref{prop.bounded_pts}). The boundedness of $ P_s $ by $ P_o $ is defined as  
\begin{equation}
    \mathcal{B}_{P_s \circ P_o} \coloneqq \frac{N_{s,b}}{N_s}.
    \label{eq.boundedness_skeletal_points}
\end{equation}
For curve skeleton $G_s$, we evenly sample $N_{sp}$ points from curve skeleton $G_s$ according to Definition~\ref{def.curve_skeleton_point}, and $N_{sp,b}$ is the number of the bounded points. A similar definition of the curve skeleton boundedness can be given by
\begin{equation}
    \mathcal{B}_{G_s \circ P_o} \coloneqq \frac{N_{sp,b}}{N_{sp}}.
    \label{eq.boundedness_curve}
\end{equation}
\end{definition}

Accurately computing the area covered by points on a sphere surface is challenging. To simplify the problem, we project the sphere points onto a 2D plane to enable Delaunay triangulation in a single step.  

Given the coordinates of an arbitrary point on the unit sphere $ [x_s, y_{s}, z_{s}]^T $, its corresponding coordinates on the 2D plane are given with sinusoidal projection method by \cite{seong2002sinusoidal} 
\[
    \left\{
        \begin{aligned}
        x_p &= \sqrt{x_s^2 + y_{s}^2} \cdot \tan^{-1}(y_{s}/ x_s), \\  
        y_p &= r \cdot \tan^{-1}\left(z_{s}/\sqrt{x_s^2 + y_{s}^2}\right)
        \end{aligned}
    \right.
\]
where $ [x_p, y_p] $ are the planar coordinates corresponding to $ [x_s, y_{s}, z_{s}] $.  

After projecting the points onto the 2D plane, we apply Delaunay triangulation. The triangulation results can then be mapped back to the sphere surface to compute the filled areas, denoted as $S_{x_i,P}$ in Eq.~(\ref{eq.boundedness_pts}). The total filled area is approximated by the sum of the areas of all triangles, as illustrated in Fig.~\ref{fig.tri_sphere}. 
\begin{prop}
\label{prop.bounding_judgement}
Assuming that the shape $\Omega\subset\mathbb{R}^3$ with boundary $\partial \Omega$ represented by $P_o$ is a pure convex shape. $N_o$ is the number of points of $P_o$. As shown in Fig.~\ref{fig.tri_sphere}, let $S_{tri}$ be the approximation of $\hat{S}_{x^\prime,P_o}$ measured by the sum of all triangle areas. Then $\hat{\beta}_{x^\prime,P_o}=\hat{S}_{x^\prime,P_o}/(4\pi r^2)$ is an approximation of $\beta_{x^\prime\circ P_o}$ in Eq.~(\ref{eq.boundedness_pts}). A point $ x^\prime \in X $ is considered to be bounded by $ \partial\Omega_o $ if $ \hat{\beta}_{x^\prime\circ P_o} < \beta^* $, where $ 0<\beta^*\leq 1 $ is a boundedness threshold dependent on the point density of $ P_o $. If $N_o\to \infty$, $\beta^*\to 1$ results an accurate approximation of boundedness.
\label{prop.bounded_pts}
\end{prop}

\begin{proof}
As illustrated in Fig.~\ref{fig.tri_sphere}, let $ a_i $ denote the area of each triangle formed by a point $ x^\prime $ and the point set $ P_o $ (Definition~\ref{def.boundedness_pts}). Let $ a^* = \sup a_i $ represent the maximum area among all such triangles. When $ N \to \infty $, we have $ a^* \to 0 $ and $ \hat{S}_{x^\prime, P_o} \to S_{x^\prime, P_o} $, where $ \hat{S}_{x^\prime, P_o} $ is the approximate area and $ S_{x^\prime, P_o} $ is the exact area. 

If $ x^\prime $ is completely bounded by $ \partial \Omega $, it follows that $ \hat{S}_{x^\prime, P_o} \to S_{x^\prime, P_o} = 4\pi r^2 $, the total surface area of the unit sphere. Consequently, $ \beta^* \to 1 $ as $ \hat{S}_{x^\prime, P_o} \to 4\pi r^2 $.
\end{proof}

\begin{figure}[t!]
\centering
    \includegraphics[width=0.2\textwidth]{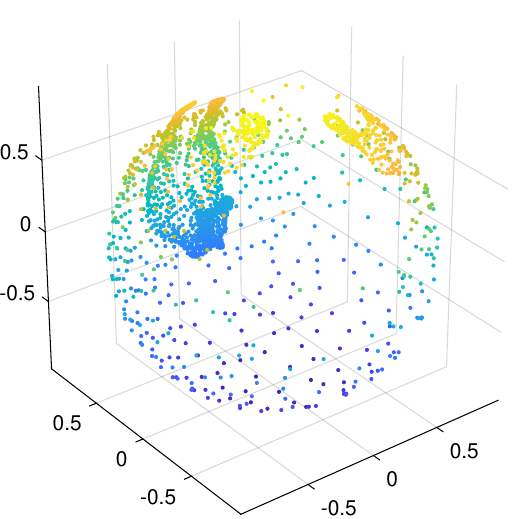}
    \caption{Distance vectors projected to a ball surface.}
    \label{fig.bound_proj}
\end{figure}

\begin{figure}[t!]
\centering
    \subfloat[Bounded point]{\includegraphics[width=0.16\textwidth]{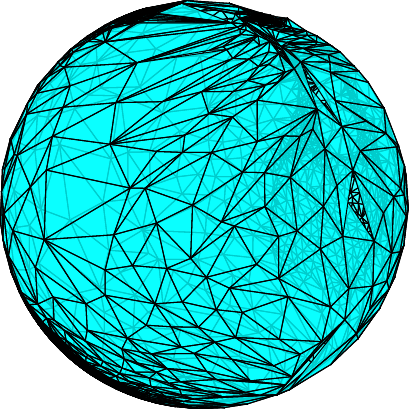}}
    \hspace{0.01\textwidth}
    \subfloat[Unbounded point]{\includegraphics[width=0.16\textwidth]{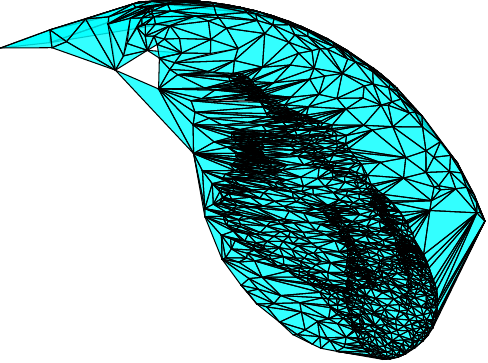}}\\
    \caption{Triangulation of the projected points on sphere surface of a point bounded and a point unbounded by a point cloud shape respectively. The tested point cloud is of a horse shape.}
    \label{fig.tri_sphere}
\end{figure}

\begin{figure}[t!]
\centering
    \subfloat[Colored by Boundedness ]{\includegraphics[width=0.23\textwidth]{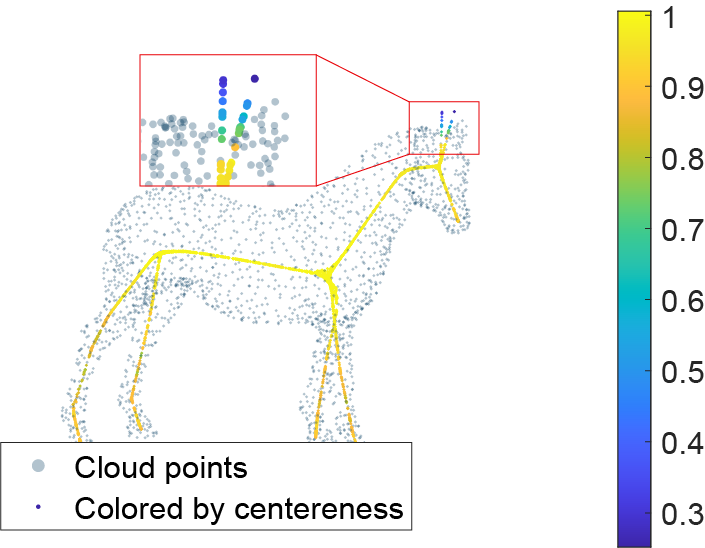}}
    \hspace{0.01\textwidth}
    \subfloat[Boundedness Distribution]{\includegraphics[width=0.23\textwidth]{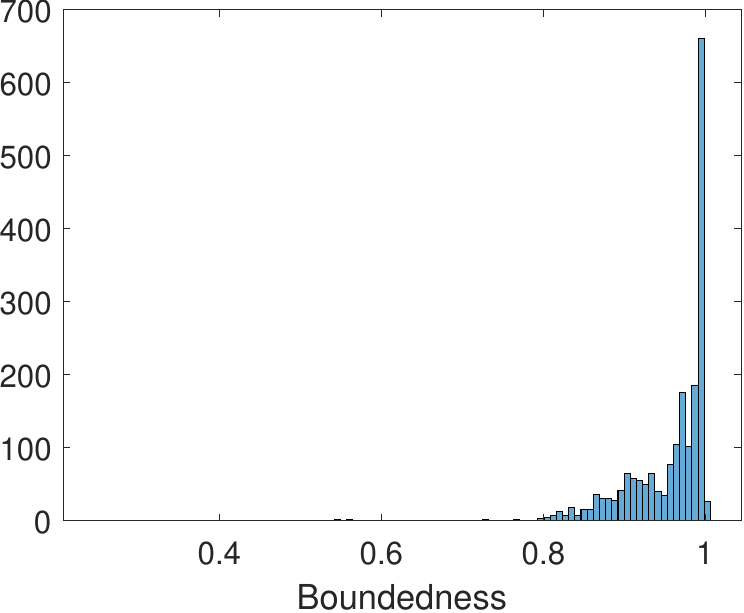}}\\
    \caption{Boundedness of contracted horse skeletal points}
    \label{fig.horse_boundedness}
\end{figure}

In practice, as shown in Fig.~\ref{fig.horse_boundedness}, the point density is often not as dense as theoretically expected. Consequently, $ \beta^* < 1 $ is the typical value observed in practical scenarios. Furthermore, it is demonstrated that the points of the horse shape lying outside the boundary exhibit smaller boundedness values, whereas the points within the boundary have boundedness values that are very close to 1.

\subsection{Centeredness}
Centeredness is the most critical characteristic for point cloud skeletonization, as it is the key feature for application. However, quantifying this feature is quite challenging since the object surface represented by a point cloud is usually unknown in practice and the skeleton definitions are not unified in previous studies. Addressing the challenges, this subsection proposes a metric for evaluating the centeredness of skeletonization results based on both the definition of surface skeletons and curve skeletons.

\subsubsection{Problem Statement \& Preliminaries}
Multiple definitions have been proposed for desirable skeletons due to differing interpretations of skeletonization. However, centeredness, widely regarded as a key feature, lacks a quantitative definition~\cite{cornea_curve-skeleton_2007, sobiecki2014comparison}. Moreover, no commonly accepted definition of centeredness exists~\cite{tagliasacchi20163d}. The most widely accepted characteristic of skeleton centeredness is that both surface skeletons and curve skeletons are expected to approximate the medial axis of the shape~\cite{palagyi2006quantitative, giesen2009scale, sobiecki2014comparison, telea2012computing}.

According to previous studies~\cite{cornea_curve-skeleton_2007,sobiecki2014comparison}, the strict definition of centeredness of both the skeletal point set $P_s$ and curve skeleton $G_s$ of a shape $\Omega$ can be as follows. A skeletal point set $ P_s =\{ \mathbf{p}_{s,i} \mid \mathbf{p}_{s,i} \in \mathbb{R}^3 \}$ corresponding to the shape $\Omega$ is expected to have each skeletal point $ \mathbf{p}_{s,i} $ lie on the medial axis $ \mathcal{M}_\Omega $ will be defined by (Eq.~\ref{eq.medial_axis}). Similarly, a curve skeleton $ G_s $ is represented as a graph structure $ (V, E) $ corresponding to the shape $\Omega$. Both the vertices $ \mathbf{v}_i \in V $ and every point $ e_i(t) \in E $ along the edges are expected to lie on the medial axis $ \mathcal{M}_\Omega $ of the shape $\Omega$, where $ t $ is the parametrization of the edge.

It is noteworthy that exact centeredness is not always necessary since it may center in some dimensions but not in others and the requirement for centeredness varies in application scenarios. For example, approximate centeredness is acceptable for virtual navigation, as the medial axis may contain excessive and unnecessary details~\cite{cornea_curve-skeleton_2007}. Moreover, the above definition assumes a known shape surface, which is not available for point cloud data. To base our metric, we have the following definitions of medial axis for known shape surfaces before extending to our numerical definition and approximation method of centeredness for point cloud shape.
\begin{definition}
    \label{def.medial_axis}
Given a shape $\Omega\subset \mathbb{R}^3$ with boundary $\Omega$, the distance transform of the shape ${\rm DT}_{\partial\Omega}:\mathbb{R}^3\to\mathbb{R}_{>0}$ are defined as
\[
    {\rm DT}_{\partial\Omega}(\mathbf{p}_x\in\Omega)=\min_{\mathbf{p}_y\in\partial\Omega}{\|\mathbf{p}_x-\mathbf{p}_y\|}
\]
\end{definition}

\begin{definition}
The medial axis or medial axis surface $S_\Omega$ of a shape $\Omega$ is a surface within the shape surface and constrained by the distance transform, given by
    \begin{equation}
        \begin{aligned}
              \mathcal{M}_\Omega& = \{ \mathbf{p}_x \in \Omega \mid \exists \mathbf{m}_1, \mathbf{m}_2 \in \partial \Omega, \mathbf{m}_1 \neq \mathbf{m}_2,\\
              &\|\mathbf{p}_x - \mathbf{m}_1\| = \|\mathbf{p}_x - \mathbf{m}_2\| \}.
        \end{aligned}
        \label{eq.medial_axis}
    \end{equation}
\end{definition}

\subsubsection{Methodology}
Since the surface of a shape captured by a real-scanned point cloud is usually uncertain and the point cloud is always with noises, it is hard to quantify the exact centeredness defined with the medial axis (Definition~\ref{def.medial_axis}). However, it is possible to estimate how close the resultant skeleton is to the medial surface with the assumption that the opposite points in the shape section are gathered together with the effect of expected skeletonization.

For skeletal point sets, centeredness can be measured under the assumption that each point in the skeletal point set $P_s$ has a corresponding match in the original point cloud $P_o$. Note that if a direct point-wise correspondence is unavailable, it can be approximated through reverse projection or by mapping the skeleton onto a potential neighborhood of object cloud points.
\begin{definition}
     Let $\mathbf{p}_{s,i}\in P_s$ be a point of $P_s$. The $k$ nearest neighbor points of $\mathbf{p}_{s,i}$ are computed by K-Nearest Neighbors (KNN) Algorithm and denoted as $\{\mathbf{p}_{s,i\odot j}\}$, where $j=1,\cdots,k$. Let $\varphi$ be the mapping from a point $\mathbf{p}_{s,i}$ of $P_s$ to its corresponding point $\mathbf{p}_{o,i}$ of $P_o$. The centeredness of a skeletal point $\mathbf{p}_{s,i}$, which represents how centrally it is positioned relative to the overall shape of the point cloud, is given by
    \begin{equation}
        c(\mathbf{p}_{s,i})\coloneqq1-\frac{\|\sum_{j=1}^{k}{\varphi(\mathbf{p}_{s,i\odot j})}-\sum_{j=1}^{k} \mathbf{p}_{s,i\odot j}\|}{\sum_{j=1}^{k}{\|\varphi(\mathbf{p}_{s,i\odot j})-1/k\cdot\sum_{j=1}^{k}{\mathbf{p}_{s,i\odot j}}\|}}
        \label{eq.skeletal_point_centeredness}
    \end{equation}
\label{def.point_centeredness}
\end{definition}

\begin{rem}
As illustrated in Fig.~\ref{fig.skeleton_centeredness}, the centeredness of skeletal points given by Eq. (\ref{eq.skeletal_point_centeredness}) generated by the Laplacian-based skeletonization method~\cite{cao_point_2010} indicates that points moved toward the center exhibit higher centeredness values. However, this method fails to adequately address points representing the joints and endpoints of the skeleton. Furthermore, the centeredness calculated using this method becomes meaningless for points that move outside the shape, such as the skeletal points representing the horse's ears in Fig.~\ref{fig.skeleton_centeredness}(a).
\end{rem}

\begin{figure}[t!]
\centering
    \subfloat[Skeletal points]{\includegraphics[width=0.23\textwidth]{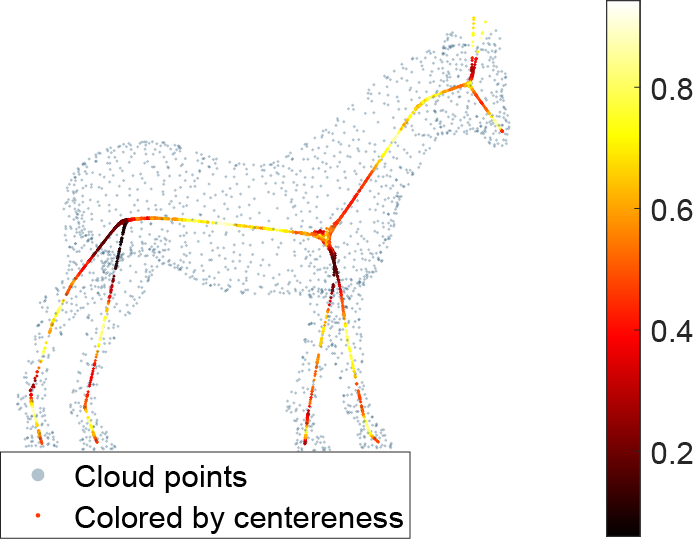}}
    \hspace{0.01\textwidth}
    \subfloat[Points on the curve]{\includegraphics[width=0.23\textwidth]{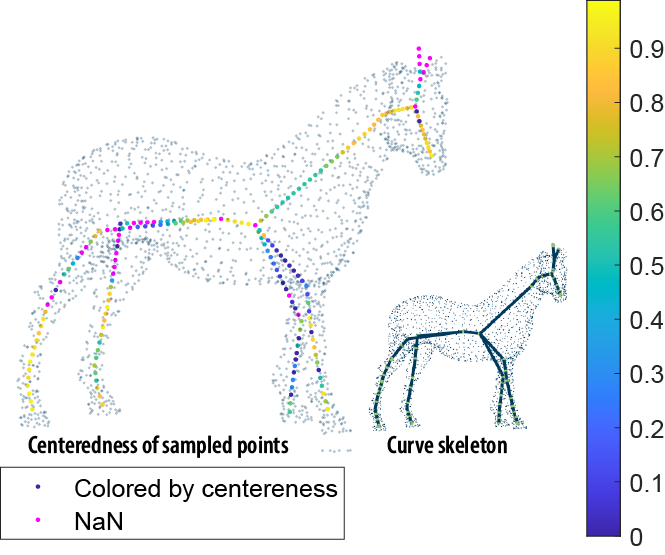}}\\
    \caption{Centeredness of skeletal points (a) and the points sampled from curve (b) respectively.}
    \label{fig.skeleton_centeredness}
\end{figure}

Because the physical and geometric properties of a curve skeleton (curve represented by continuously connected vertices rather than a set of discrete points) are different from surface skeleton, we consider for a curve skeleton $ G_s = (V, E) $, the centeredness of a point on the curve is determined by its associated original cloud points, which are the points in the original cloud near the curve point in same space. As illustrated in Fig.~\ref{fig.shape_cutting}, the original cloud points can be separated by two cutting planes $ M_1 $ and $ M_2 $, provided the planes have the same normal vector and the normal vector aligns with the curve's direction.
\begin{definition}
\label{def.associated_points}
    Consider two neighboring edges $ e_1 = (\mathbf{v}_0, \mathbf{v}_1) $ and $ e_2 = (\mathbf{v}_1, \mathbf{v}_2) $ in the graph $ G_s $, where $ e_1, e_2 \in E $ and $ \mathbf{v}_0, \mathbf{v}_1, \mathbf{v}_2 \in V $.

    Let $ \mathbf{p}_{g,1} = e_1(t_1) $ be a point on the edge $ e_1 $, where $ t_1 \in [0, 1] $ (Definition~\ref{def.curve_skeleton_point}). The direction vector of the curve at $ \mathbf{p}_{g,1} $ is denoted as $\mathbf{u}_{1}$ and determined as follows:
    \begin{enumerate}
        \item For $ t_1 \in (0, 1) $: $\mathbf{u}_{1}$ is given by  
        \[
            \mathbf{u}_{1}=\frac{\mathbf{v}_1 - \mathbf{v}_0}{\|\mathbf{v}_1 - \mathbf{v}_0\|}.
        \]
        \item For $ t_1 = 1 $: $\mathbf{u}_{1}$ is approximated by the tangent direction at $ \mathbf{v}_1 $, computed as the tangent to the circle determined by the coordinates of $ \mathbf{v}_0 $, $ \mathbf{v}_1 $, and $ \mathbf{v}_2 $.
    \end{enumerate}
    The associated points $ Q_{1} \subset P_o $ corresponding to $\mathbf{p}_{g,1}$ are defined as the points enclosed by two parallel cutting planes. The orientation and position of these planes are determined by the curve direction and the coordinates of $\mathbf{p}_{g,1}$, with an interval value $\epsilon_p$ between them.
\end{definition}

With the approximation of the direction of a point on the curve and given interval $\epsilon_p$ between the planes, the two cutting planes are confirmed and can be used to separate points from the original point cloud, as shown in Fig.~\ref{fig.shape_cutting}. The value of $\epsilon_p=\alpha\inf\|\mathbf{v}_i-\mathbf{v}_j\|, \alpha\in(0,1)$ is dependent on the minimum value of two neighboring vertices of $G_s$.

After retrieving the associated points of a point $\mathbf{p}_{g,i}$ of curve skeleton $G_s$ from original point cloud, the center of those associated points is required for estimating the centeredness of $\mathbf{p}_{g,i}$.

\begin{definition}
    Let $Q_{i}\subset P_o\subset\mathbb{R}^{n\times3}$ be the associated points (Definition~\ref{def.associated_points}) of $\mathbf{p}_{g,i}$, where $\mathbf{p}_{g,i}\in G_s$ is a point sampled from curve skeleton $G_s$ (Definition~~\ref{def.curve_skeleton_point}). Two orthogonal basis perpendicular to $\mathbf{u}_{i}$ (Definition~\ref{def.associated_points}), the direction of $G_s$ at $\mathbf{p}_{g,i}$ are denoted as $\mathbf{g}_{i},\mathbf{h}_{i}$ respectively. The projected 2D points $\hat{Q}_{i}$ is given by
    \[
        \hat{Q}_{i}=Q_i\cdot[\mathbf{g}_{i},\mathbf{h}_{i}].
    \]
    Correspondingly, the projected $\mathbf{p}_{g,i}$ is given by $\hat{\mathbf{p}}_{g,i}=\mathbf{p}_{g,i}\cdot[\mathbf{g}_{i},\mathbf{h}_{i}]$. The center of the fitted ellipse of the projected points in $\hat{Q}_{i}$ is denoted by $\hat{\mathbf{q}}_c$. The centeredness of $\mathbf{p}_{g,i}$ is then given by
    \begin{equation}
        \mathfrak{c}(\mathbf{p}_{g,i})\coloneqq1-\frac{\|\hat{\mathbf{p}}_{g,i}-\hat{\mathbf{q}}_c\|}{0.5\cdot(l_a+l_b)}
        \label{eq.curve_point_centeredness}
    \end{equation}
    where $l_a,l_b$ are the length of the semi-major axis and the semi-minor axis of the fitted ellipse respectively. To make sure $\mathbf{p}_{g,i}$ is non-negative, if $\mathfrak{c}(\mathbf{p}_{g,i})<0$, $\mathfrak{c}(\mathbf{p}_{g,i})$ will be reassigned as 0.
\end{definition}

\begin{rem}
    Inspired by the work of Fitzgibbon et al.~\cite{fitzgibbon1999direct}, the center of the projected points $\hat{Q}_{i}$ can be robustly estimated using ellipse fitting. As shown in Fig.~\ref{fig.shape_cutting}, the center estimated via ellipse fitting is closer to human visual estimation than the barycenter obtained by averaging the positions of all points. Experimental results on simple geometries further confirm that ellipse fitting provides a more robust center estimation, as illustrated in Fig.~\ref{fig.ellipse_fitting}, where the geometrical centers of 2D point sets from multiple shapes are shown.

    As illustrated in Fig.~\ref{fig.skeleton_centeredness}(b), points are evenly sampled from the curve skeleton, and centeredness is calculated and visualized using color. Not every point on the curve has a valid centeredness value. Points without valid centeredness values are marked in magenta, including those coinciding with joint vertices of the curve skeleton and points with fewer than three associated points. A joint vertex in a curve skeleton $G_s$ is defined as a vertex connected to more than two edges.
\end{rem}

\begin{figure}[t!]
\centering
    \includegraphics[width=0.35\textwidth]{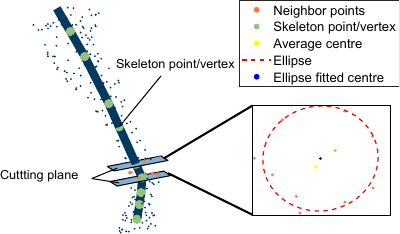}
    \caption{Separated points (orange) by two cutting planes.}
    \label{fig.shape_cutting}
\end{figure}

\begin{definition}
\label{def.overall_centeredness}
    The overall quantified centeredness of either the skeletal points or the curve skeleton is defined as the ratio of elements that are sufficiently centered. A skeletal point $\mathbf{p}_{s,i} \in P_s$ is considered sufficiently centered if its centeredness value $c(\mathbf{p}_{s,i})$, as given in (\ref{eq.skeletal_point_centeredness}), satisfies $c(\mathbf{p}_{s,i}) > c^*$. Similarly, a sampled point $\mathbf{q}_{e}$ of the curve skeleton $G_s$ is sufficiently centered if its centeredness value $\mathfrak{c}(\mathbf{q}_{e})$, as defined in (\ref{eq.curve_point_centeredness}), satisfies $\mathfrak{c}(\mathbf{q}_{e}) > \mathfrak{c}^*$. 

    Here, the threshold values $c^*$ and $\mathfrak{c}^*$ depend on the required centeredness in the specific application scenario, which is discussed in \ref{RoboticDiscussion}. By considering $N_{s,c}$ denotes the number of sufficiently centered points in $P_s$ or the number of sufficiently centered sampled points in $G_s$, and let $N_s$ denote the total number of accounted points. The overall centeredness of $P_s$ or $G_s$ is then given by:
    \begin{equation}
        \mathcal{C}_s\coloneqq \frac{N_{s,c}}{N_s}.
        \label{eq.overal_centeredness}
    \end{equation}
\end{definition}

Please note, the curve skeleton centeredness definition may also be applicable for skeletal point set if the skeletal point set is of a very thin line-like structure since in this case the principal direction of the skeletal point skeleton at a point on of the skeleton can be estimated by this point's neighbor points by Principal Component Analysis, as explained in Proposition \ref{prop.point_tangent}.

\begin{figure}[t!]
\centering
        \includegraphics[width=0.4\textwidth]{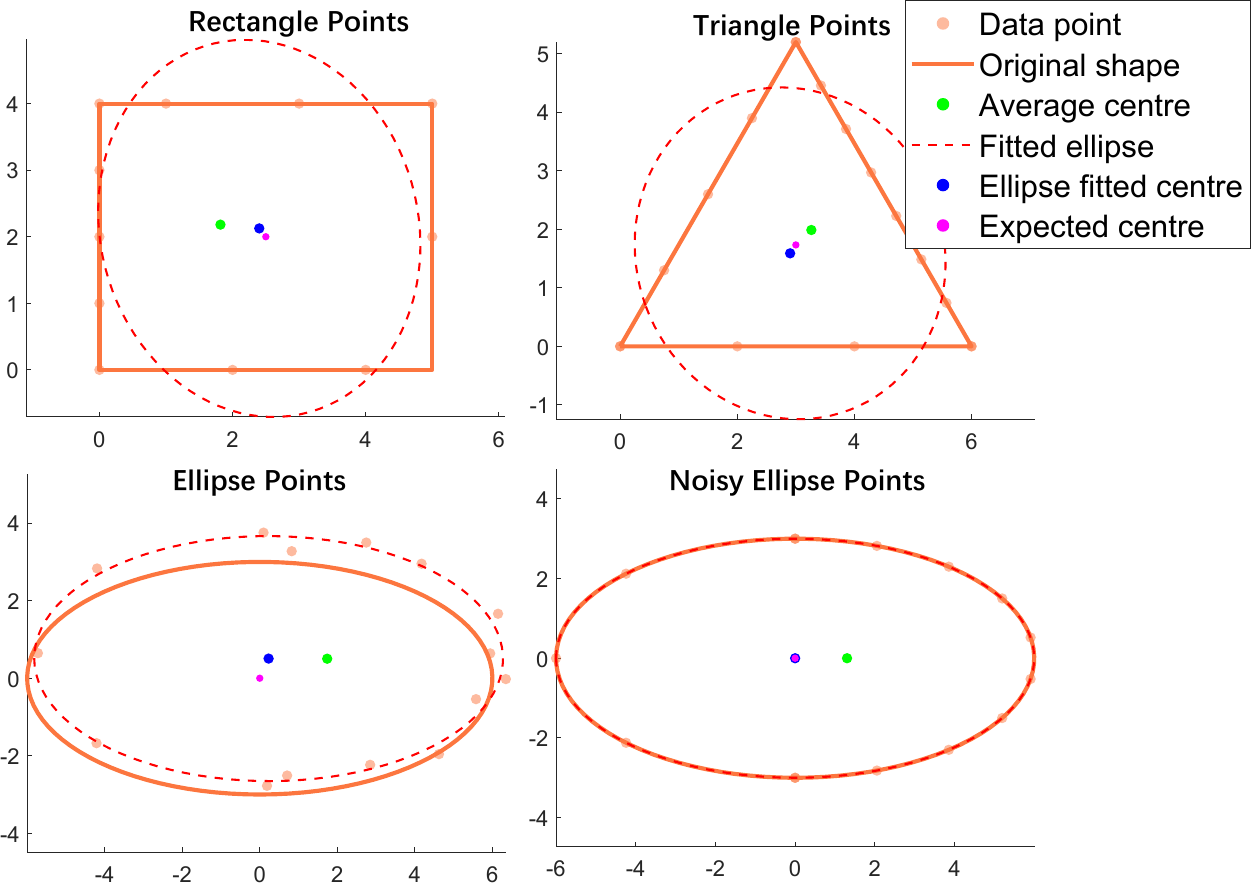}
    \caption{Ellipse fitting of the points of simple geometries.}
    \label{fig.ellipse_fitting}
\end{figure}






\subsection{Smoothness}
The smoothness of the curve skeleton is defined by the variation of the tangent direction along the curve~\cite{cornea_curve-skeleton_2007}, and this property significantly affects application scenarios such as navigation. In the work of Sobiecki et al.~\cite{sobiecki2014comparison}, it is further emphasized that both the manifold represented by the skeletal point set and the curve skeleton are expected to exhibit at least $C^2$ continuity, with curvature continuity being a desirable property. Inspired by~\cite{cornea_curve-skeleton_2007}, this subsection presents a refined mathematical definition of the smoothness metric, designed to ensure continuous differentiation. This refinement holds potential significance in path-planning problems that will be further discussed in Section \ref{RoboticDiscussion}.

\subsubsection{Problem Statement}
Since the smoothness of the skeleton affects the most on navigation and motion planning problems, in which scenario the direction change is of great importance, our problem focuses on the direction change along the curved branches of the skeleton. With the assumption that all skeletonization results, either the curve skeleton $G_s=(V,E)$ or the skeletal points $P_s$ is expected to be very thin line-like shape.

\subsubsection{Methodology}
Since the curved path of the skeleton at a given point in a point cloud model is determined by its tangent vector, and smoothness evaluation relies on measuring variations in this vector, our proposed metric addresses smoothness in three steps. First, the tangent vector is approximated for the skeletal representation, whether as a skeletal point set or a curve skeleton. Second, changes in the tangent direction along the skeleton are measured. Finally, a smoothness score is assigned based on these variations.

First, we address the estimation of the tangent vector. Given a curve skeleton $G_s=(V,E)$, the curve direction change only happens at the vertices and the tangent estimation along the curve skeleton shares with the definition in centeredness as stated in Definition~\ref{def.associated_points}. While for the curve-like skeletal points, as illustrated in Fig.~\ref{fig.bound_proj}(a), the tangent vector along the curve represented by discrete points can be estimated by the direction of the principal axis obtained through principal component analysis, as explained in Proposition~\ref{prop.point_tangent}.

\begin{prop}  
\label{prop.point_tangent}  
Given a skeletal point set $ P_s \subset \mathbb{R}^{n \times 3} $, the $ k $ nearest neighbors of a point $ \mathbf{p}_{s,i} = [x_{s,i}, y_{s,i}, z_{s,i}]^T \in \mathbb{R}^3 $ are denoted by $ \{\mathbf{p}_{s,i\odot j}\} $, where $ \mathbf{p}_{s,i}, \mathbf{p}_{s,i\odot j} \in P_s, j = 1, \dots, k $.  
The covariance matrix of these neighboring points is given by~\cite{kim2023occlusion}  
\begin{equation*}  
     \mathbf{C}_i =\frac{1}{k}\sum_{j=1}^{k}(\mathbf{p}_{s,i\odot j}-\mathbf{p}_{s,i})(\mathbf{p}_{s,i\odot j}-\mathbf{p}_{s,i})^T.  
\end{equation*}  
The tangent vector at $ \mathbf{p}_{s,i} $, denoted as $ \mathbf{t}_{\mathbf{p}_{s,i}} $, is the eigenvector corresponding to the largest eigenvalue of $ \mathbf{C}_i $.  
\end{prop}  

\begin{proof}  
If the $ k $ discrete neighboring points form a shape of short-segment, the covariance matrix $ \mathbf{C}_i $ captures its local structure. As the segment approaches a straight line, the eigenvector corresponding to the largest eigenvalue of $ \mathbf{C}_i $ aligns with the curve’s tangent direction.  
\end{proof}

Secondly, the variance of the tangent vector can be measured using cosine similarity with normalization, as given by~\cite{xia2015learning}:
\begin{equation} \label{eq.similarity} D_{n}(\mathbf t_1, \mathbf t_2)= \frac{1}{\pi} \arccos\left( \frac{\mathbf t_1 \cdot\mathbf t_2 }{\Vert \mathbf t_1 \Vert\Vert \mathbf t_2 \Vert} \right). \end{equation}
Here, $D_{n}$ quantifies the variance between two tangent vectors between two tangent vector $\mathbf t_1,\mathbf t_2$. The smoothness metric of a point in a skeletal point set and a vertex in a curve skeleton is defined in Definition~\ref{def.point_sm}.

 \begin{figure}[t!]
\centering
    \subfloat[$S_A=0.981\\S_B=0.988$]{\includegraphics[width=0.15\textwidth]{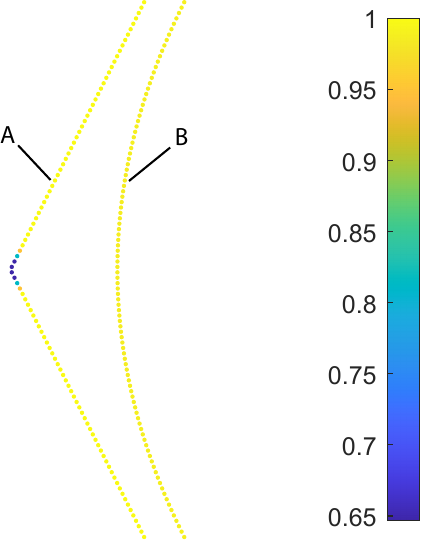}}
    \hspace{0.05\textwidth}
    \subfloat[ $\mathfrak{S}_A=0.941\\\mathfrak{S}_B=0.943$]{\includegraphics[width=0.15\textwidth]{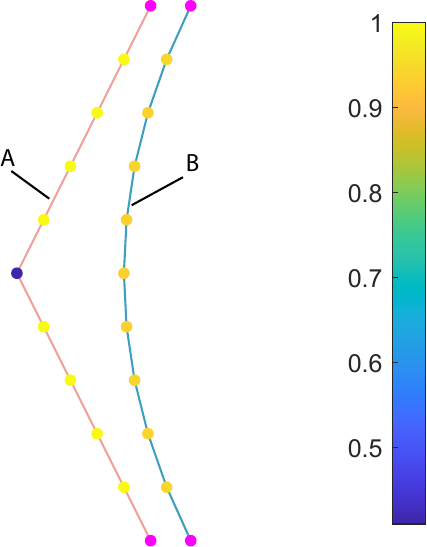}}\\
    \caption{Smoothness of 2D skeletal points (a) and 2D curve skeleton (b) in different smoothness. The vertices in (b) without valid smoothness value are marked in magenta.}
    \label{fig.skeleton_smoothness}
\end{figure}
\begin{definition}
    For a point $ \mathbf{p}_{s,i} \in P_s $ of the skeletal point set $ P_s $, the $ m $ neighbor points of $ \mathbf{p}_{s,i} $ are denoted as $ \{\mathbf{p}_{s,i \odot j}\} $. The smoothness of $ P_s $ at the point $ \mathbf{p}_{s,i} $ is defined as
    \begin{equation}
        s(\mathbf{p}_{s,i}) \coloneqq \min_{j=1,\dots,m} \lvert 1- 2\cdot{ D_n(\mathbf{t}_{\mathbf{p}_{s,i}}, \mathbf{t}_{\mathbf{p}_{s,i \odot j}}) }  \rvert,
        \label{eq.skeletal_point_smoothness}
    \end{equation}
    where $ s(\mathbf{p}_{s,i}) \in [0,1] $, and $ \mathbf{t}_{\mathbf{p}_{s,i}}, \mathbf{t}_{\mathbf{p}_{s,i \odot j}} $ are the tangent vectors at $ \mathbf{p}_{s,i} $ and $ \mathbf{p}_{s,i \odot j} $ respectively, given by Proposition~\ref{prop.point_tangent}.

    Similarly, for a vertex $ \mathbf{v}_i \in V $ of the curve skeleton $ G_s = (V, E) $, let $ N_e(\mathbf{v}_i) $ denote the number of edges connected to $ \mathbf{v}_i $. If $ N_e(\mathbf{v}_i) \neq 2 $, $ \mathbf{v}_i $ is either an endpoint, a joint point, or an isolated point without a tangent change, and the curve smoothness around the vertex is considered the maximum value 1. If $ N_e(\mathbf{v}_i) = 2 $, the smoothness of $ G_s $ at $ \mathbf{v}_i $ is given by 
    \begin{equation}
        \mathfrak{s}(\mathbf{v}_i) \coloneqq \lvert 1-2 \cdot D_n(\mathbf{v}_{i-1} - \mathbf{v}_{i}, \mathbf{v}_{i+1} - \mathbf{v}_i) \rvert,
        \label{eq.curve_point_smoothness}
    \end{equation}
    where $ \mathfrak{s}(\mathbf{v}_i) \in [0,1] $.
    \label{def.point_sm}
\end{definition}
With smoothness defined for each point of the skeletal point set and each vertex of the curve skeleton. The overall smoothness of the skeletal point set and curve skeleton
 is given by Definition~\ref{def.skelel_sm}.
\begin{definition}
    For a skeletal point set $P_s=\{\mathbf{p}_{s,i}\}$, the smoothness of $P_s$ is defined as 
    \begin{equation}
        S(P_s)=\frac{1}{N_s}\sum_{i=1}^{N_s}{s(\mathbf{p}_{s,i})},
        \label{eq.smoothness_skeletal_points}
    \end{equation}
and, for a curve skeleton $G_s=(V,E)$, the smoothness for $G_s$ is given by
    \begin{equation}
        \mathfrak{S}(G_s)=1-\frac{1}{W }\sum_{i=1}^{\hat{N}_v}{w_i\cdot(1-\mathfrak{s}(\hat{v}_i)}),
        \label{eq.smoothness_curve}
    \end{equation}
    where $\hat{v}_i$ denote all the vertices satisfy $N_e(\hat{v}_i)=2$, $w_i$ is the total length of the two half edges about $\hat{v}_i$, $W$ is the total length the all edges such that $e_i\in E$.
    \label{def.skelel_sm}
\end{definition}
\begin{figure*}
  \centering
  \begin{tabular}{ p{1.6cm}  p{2.8cm}  p{2.6cm}  p{2.6cm}  p{2.6cm}  p{2.6cm}}
    \hline
    \textbf{Categories} & \textbf{Skeletonization Results} & \textbf{Topological Similarity} & \textbf{Boundedness} & \textbf{Centeredness} & \textbf{Smoothness} \\ \hline\hline
    4096 pts
    & \includegraphics[width=0.13\textwidth]{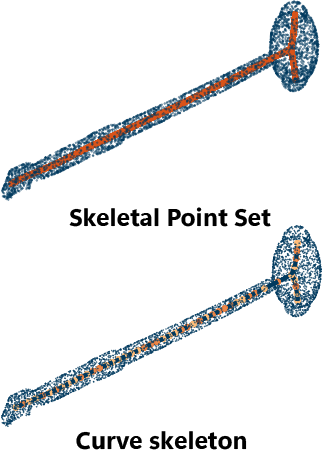} 
    & \includegraphics[width=0.13\textwidth]{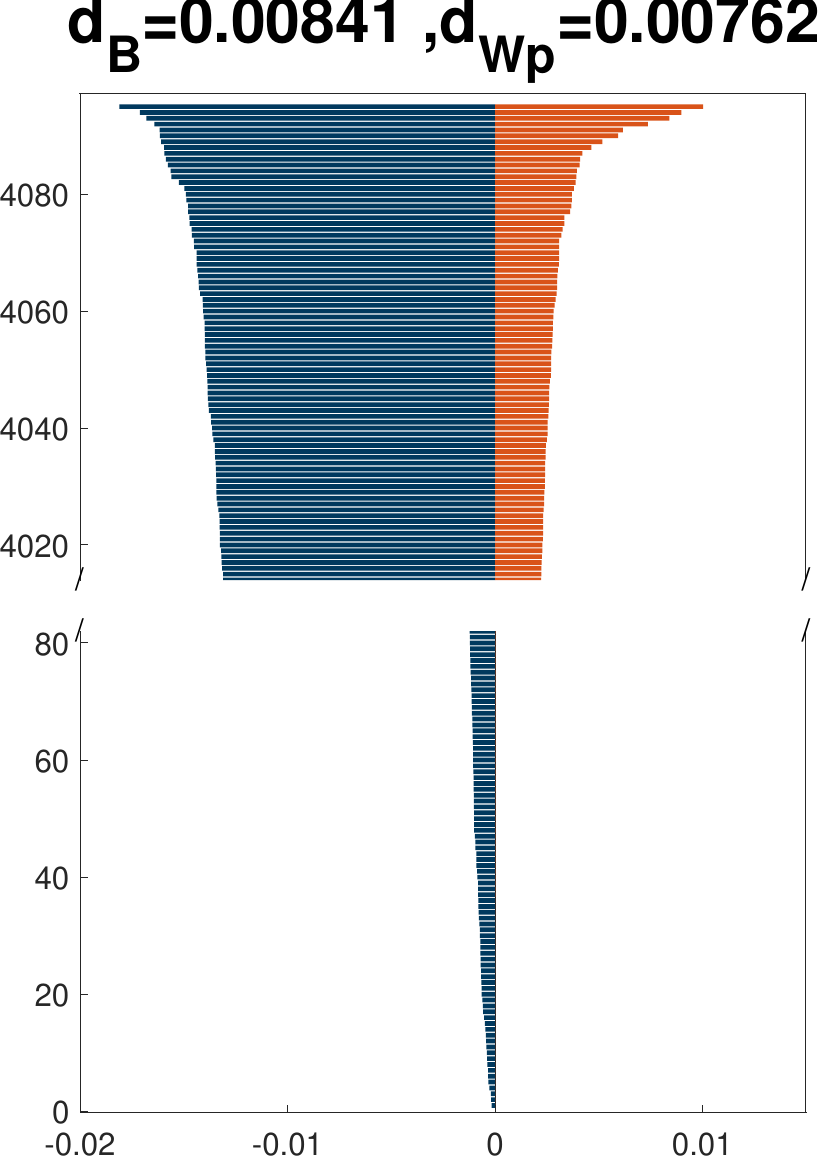}
    & \includegraphics[width=0.13\textwidth]{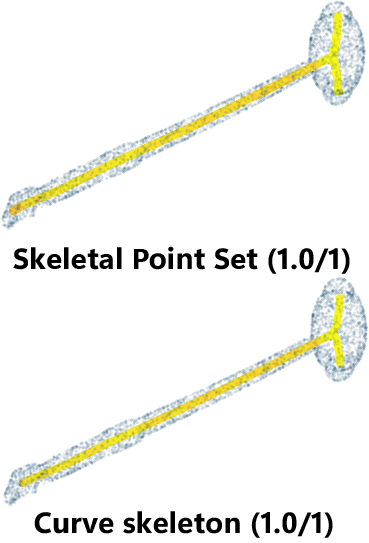} 
    & \includegraphics[width=0.13\textwidth]{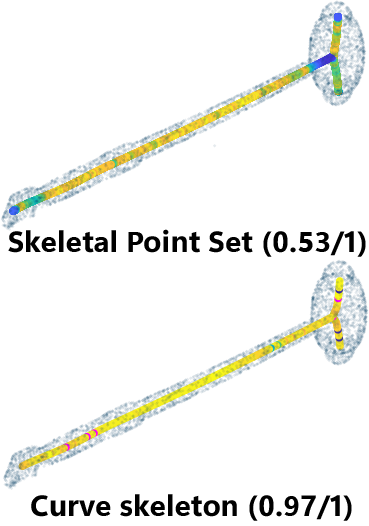}
    & \includegraphics[width=0.13\textwidth]{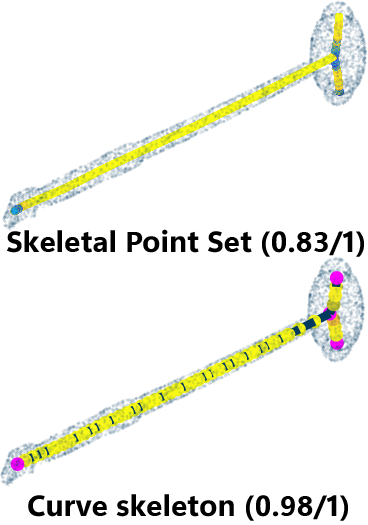} 
    \\ \cline{1-6}
    5\% noise
    & \includegraphics[width=0.13\textwidth]{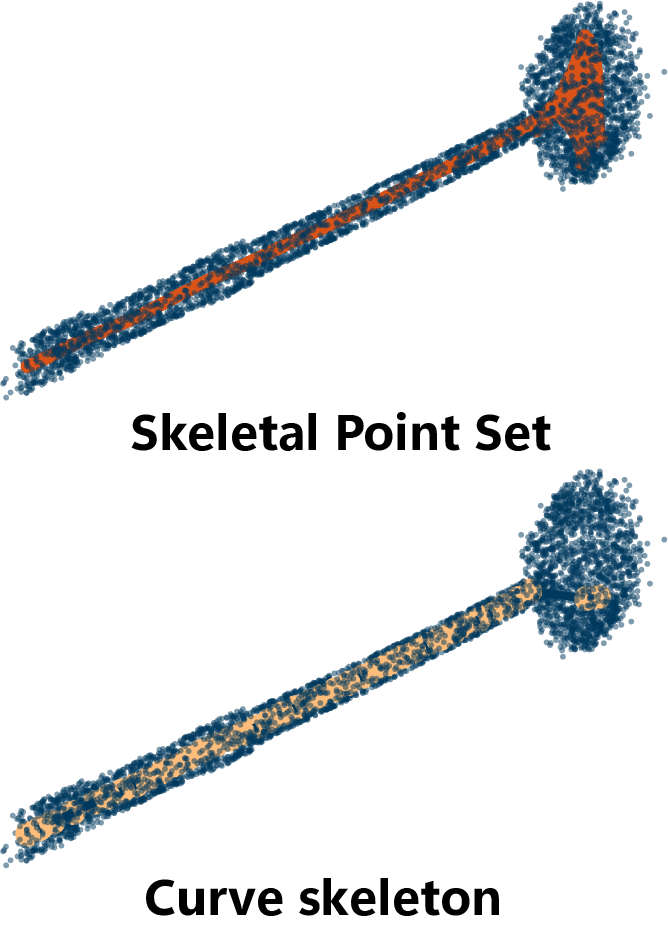}
    & \includegraphics[width=0.13\textwidth]{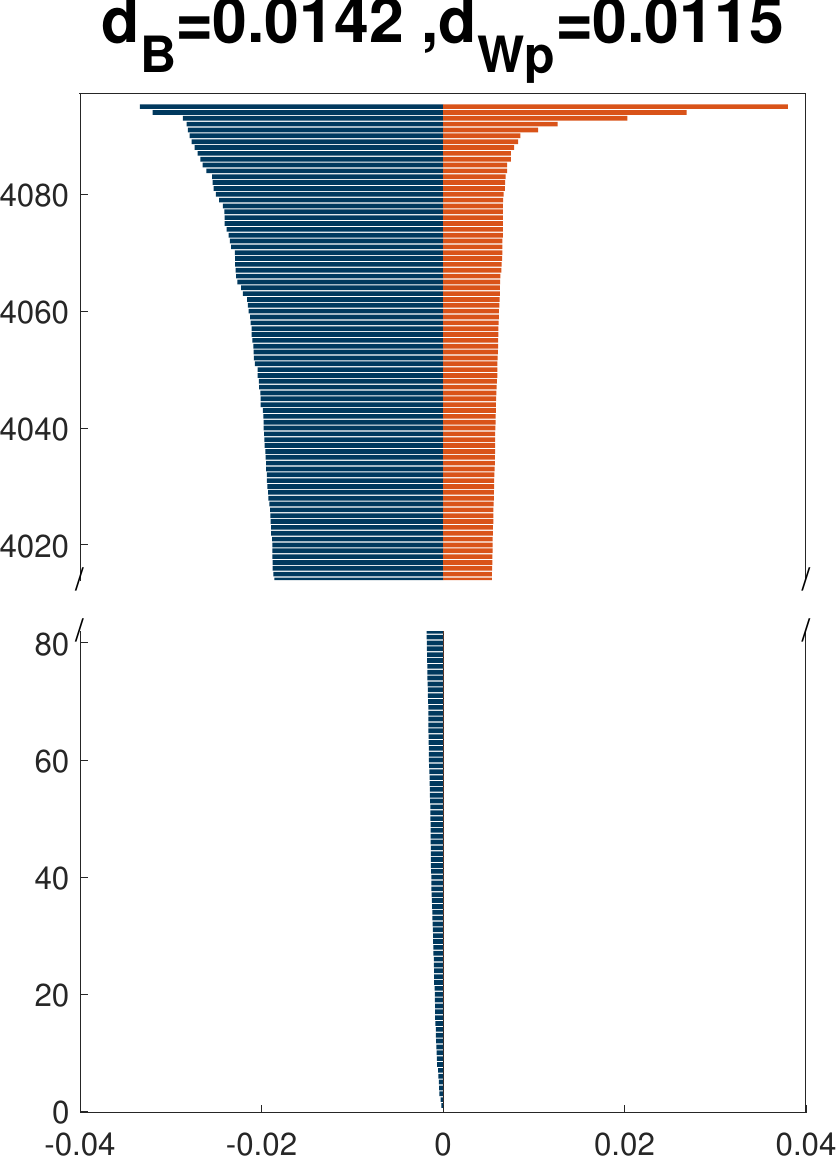}
    & \includegraphics[width=0.13\textwidth]{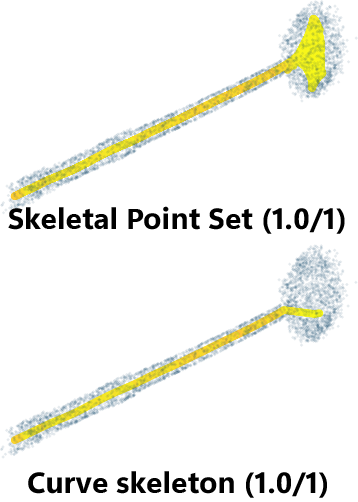}
    & \includegraphics[width=0.13\textwidth]{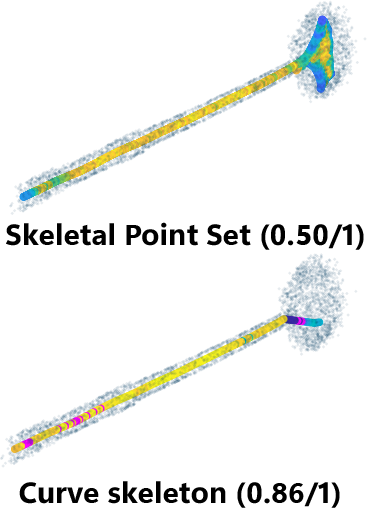}
    & \includegraphics[width=0.13\textwidth]{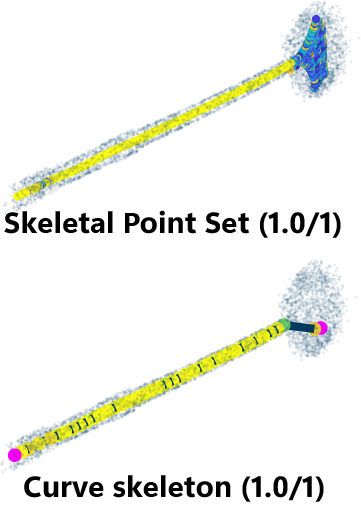} 
    \\ \cline{1-6}
    1024 pts
    & \includegraphics[width=0.13\textwidth]{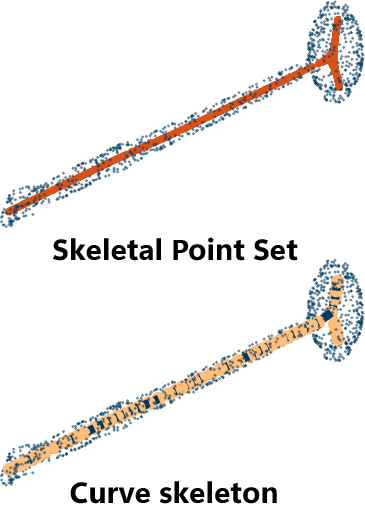}
    & \includegraphics[width=0.13\textwidth]{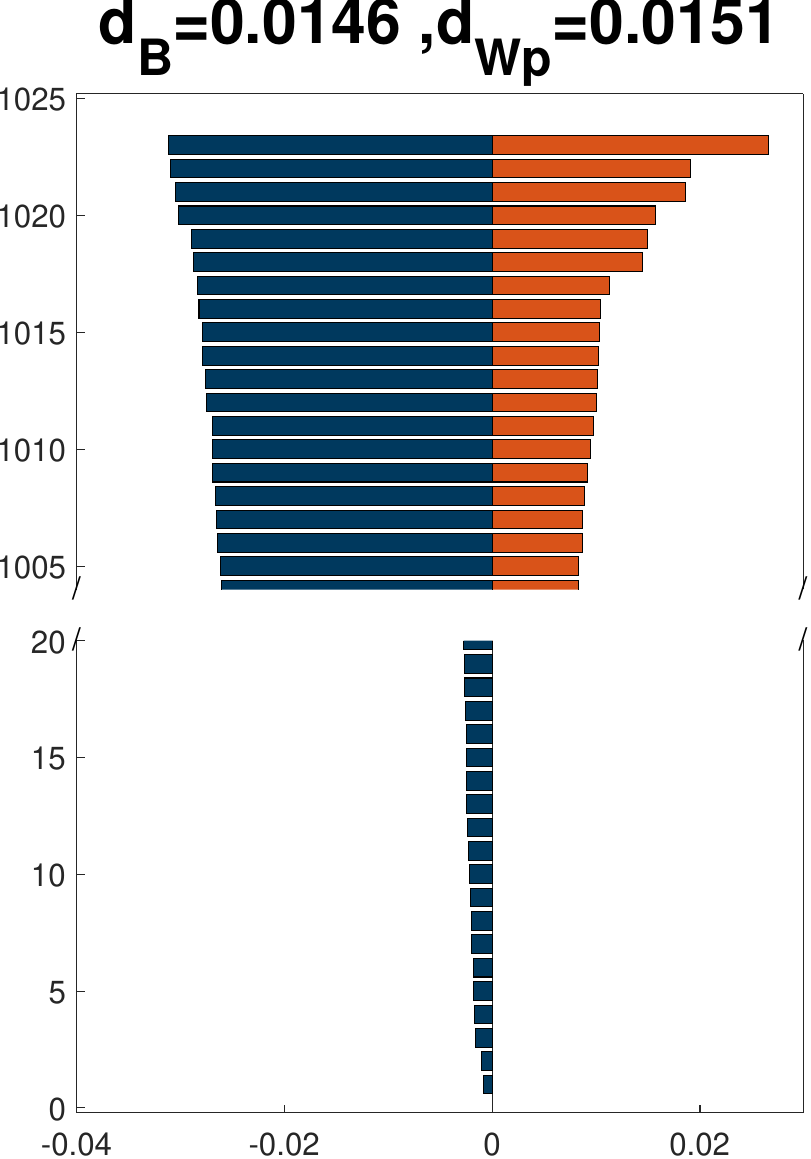}
    & \includegraphics[width=0.13\textwidth]{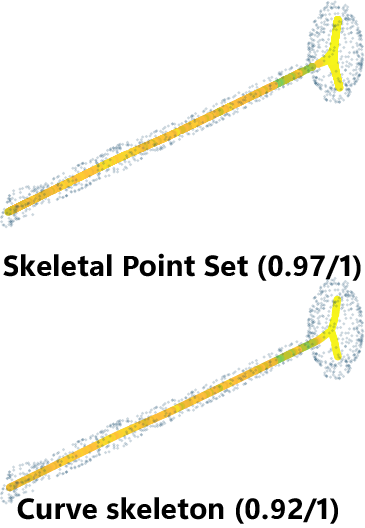}
    & \includegraphics[width=0.13\textwidth]{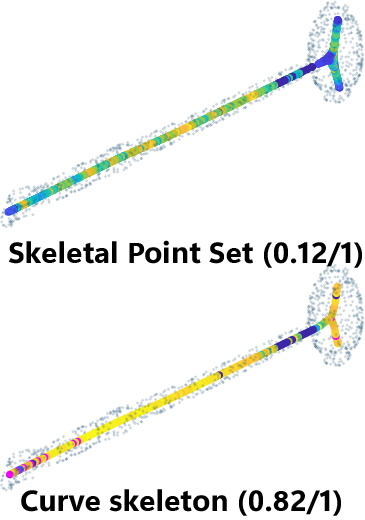}
    & \includegraphics[width=0.13\textwidth]{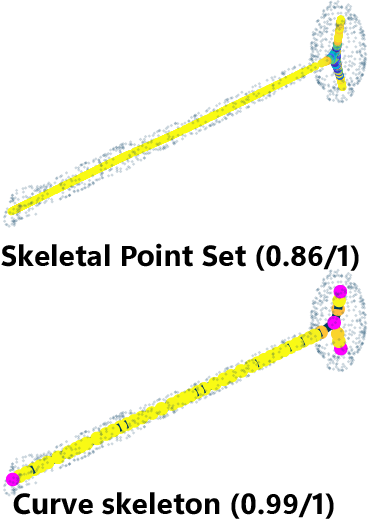}
    \\  \hline
    4096 pts
    & \includegraphics[width=0.12\textwidth]{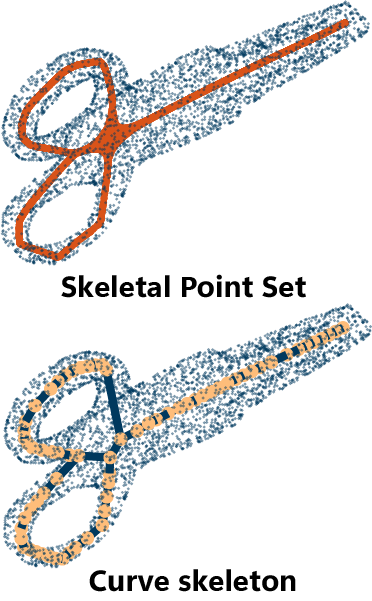}
    & \includegraphics[width=0.13\textwidth]{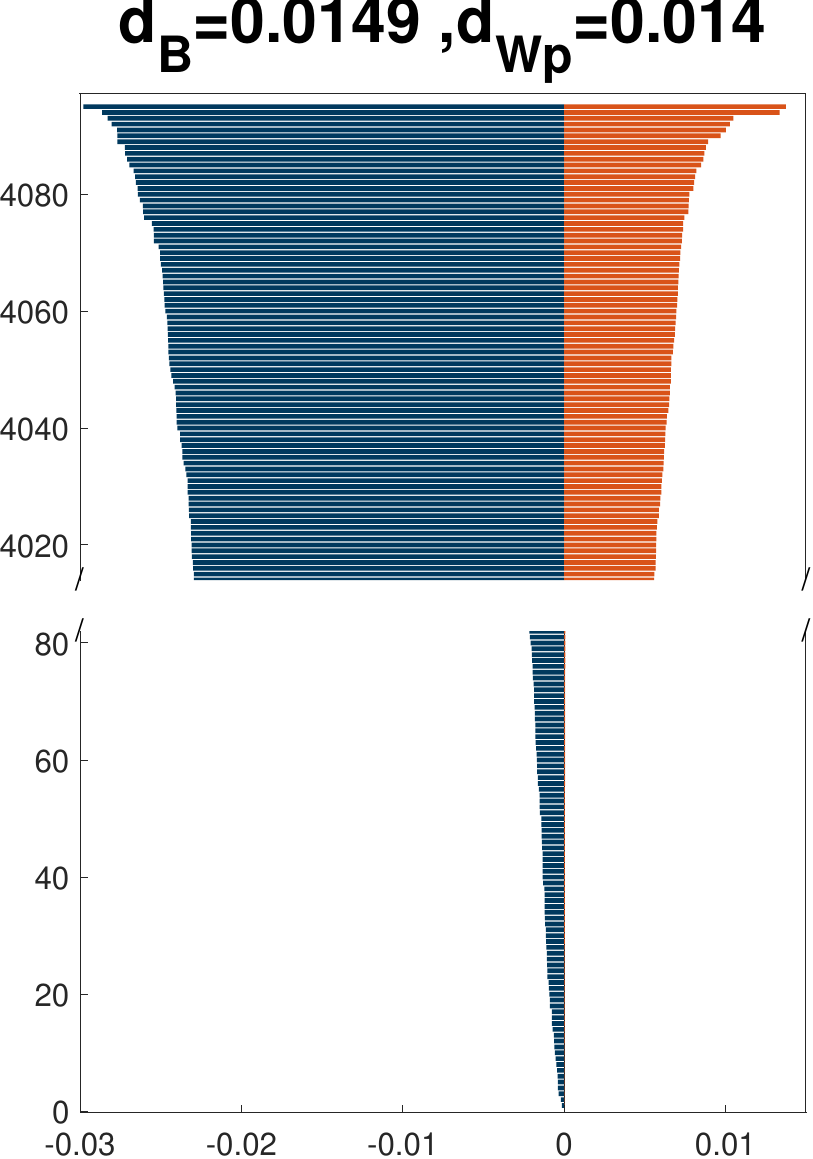}
    & \includegraphics[width=0.12\textwidth]{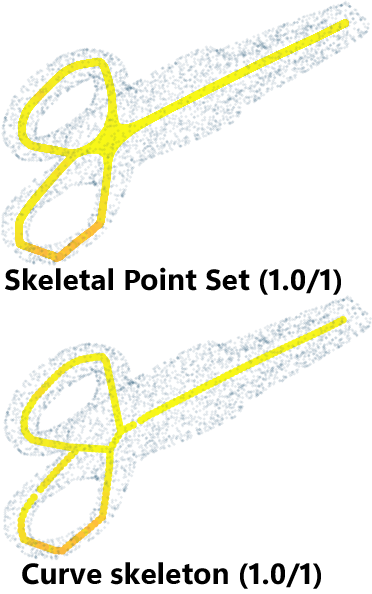} 
    & \includegraphics[width=0.12\textwidth]{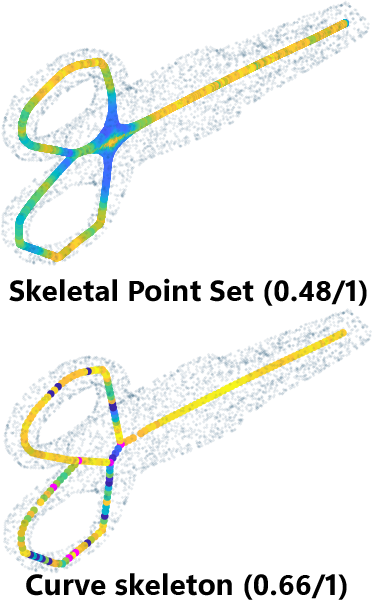}
    & \includegraphics[width=0.12\textwidth]{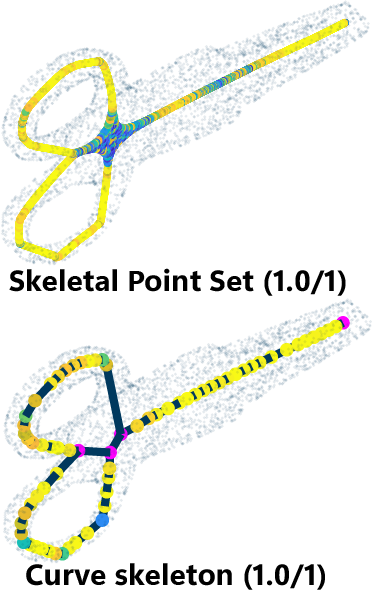} 
    \\ \hline
    4096 pts
    & \includegraphics[width=0.09\textwidth]{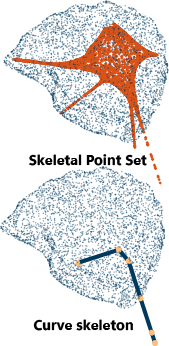}
    & \includegraphics[width=0.13\textwidth]{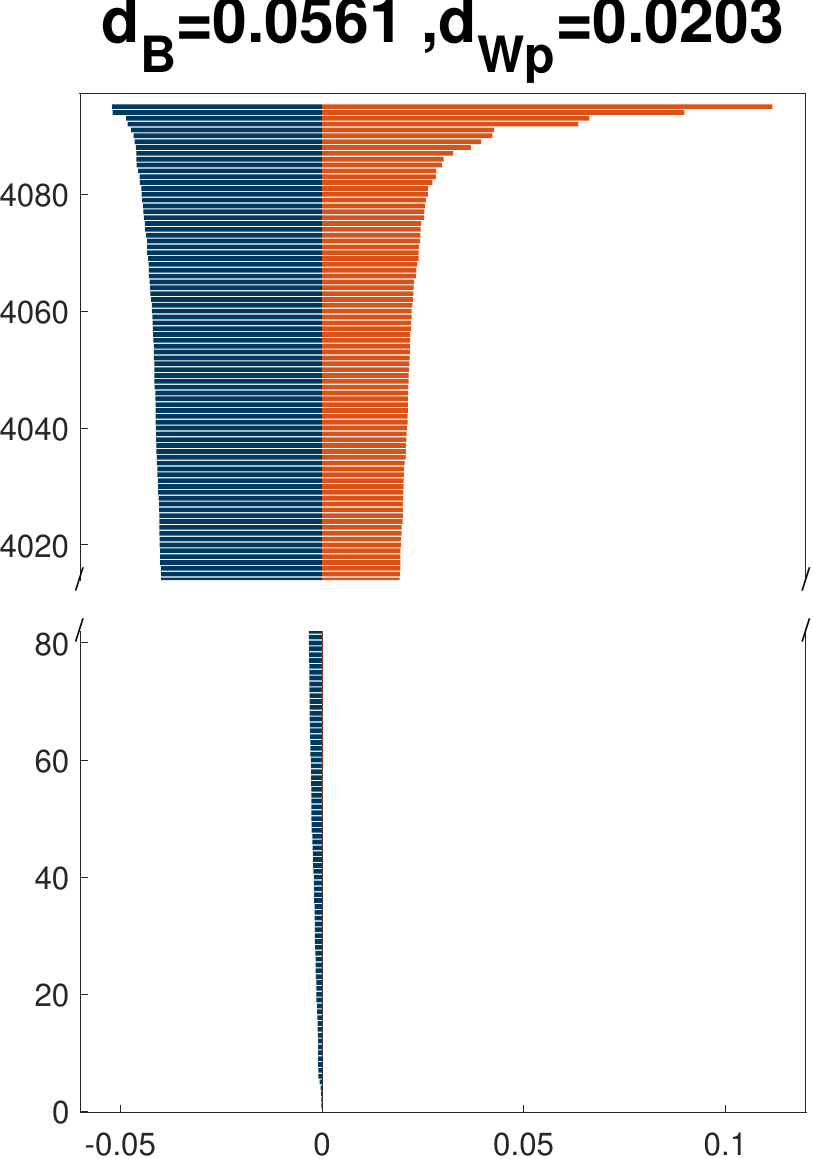}
    & \includegraphics[width=0.09\textwidth]{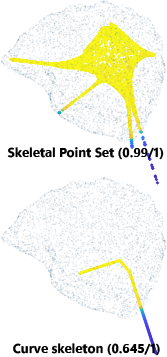}
    & \includegraphics[width=0.09\textwidth]{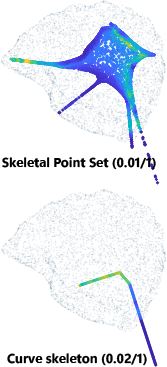}
    & \includegraphics[width=0.09\textwidth]{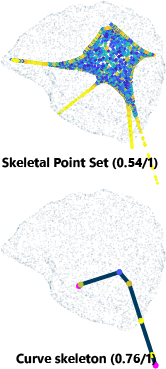} 
    \\ \hline
    4096 pts
    & \includegraphics[width=0.15\textwidth]{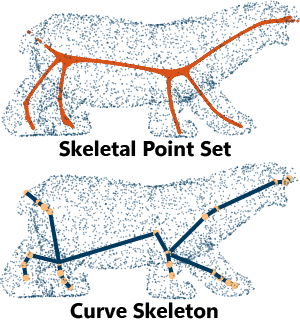}
    & \includegraphics[width=0.13\textwidth]{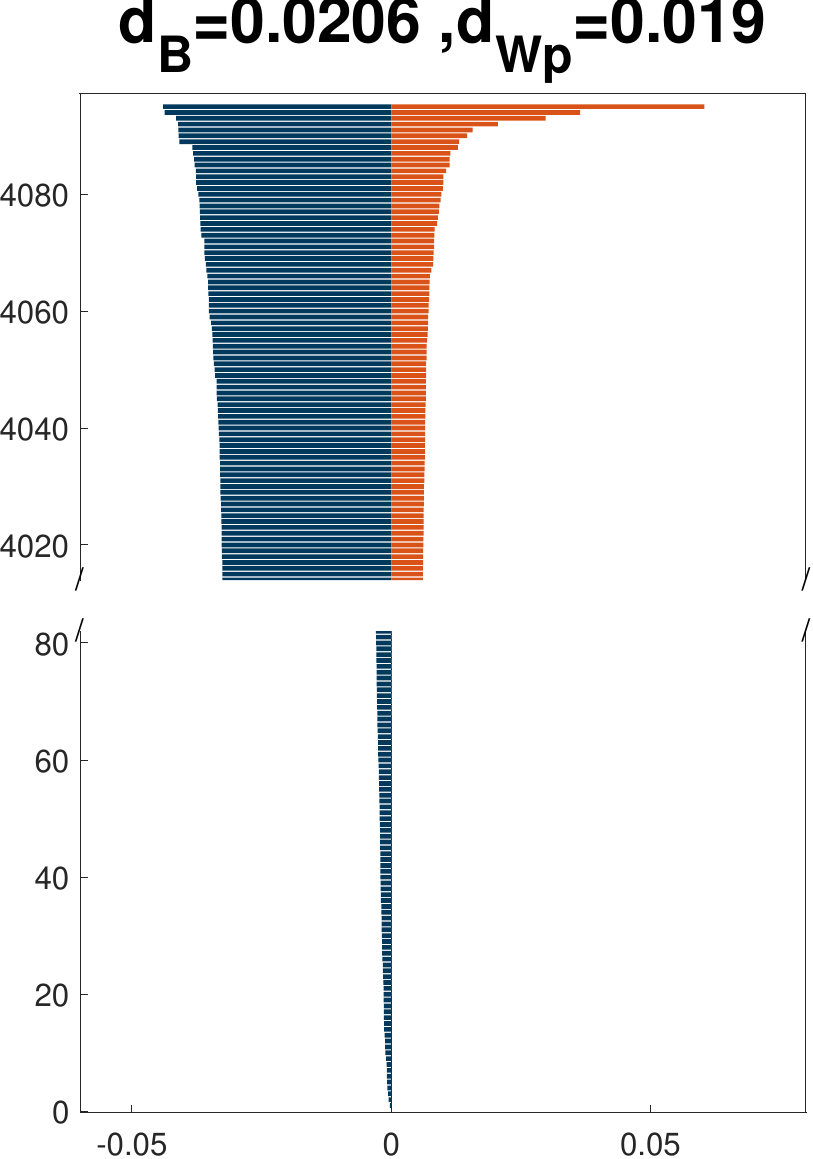}
    & \includegraphics[width=0.15\textwidth]{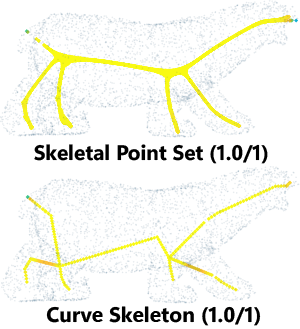}
    & \includegraphics[width=0.15\textwidth]{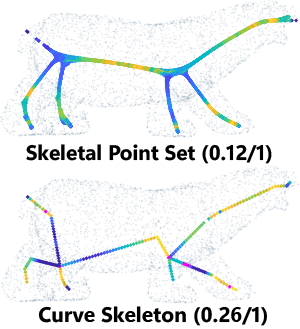}
    & \includegraphics[width=0.15\textwidth]{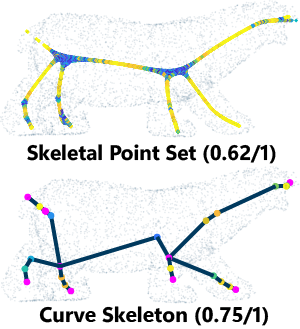} 
    \\  \hline\hline
    \multicolumn{6}{r}{\includegraphics[width=0.8\textwidth]{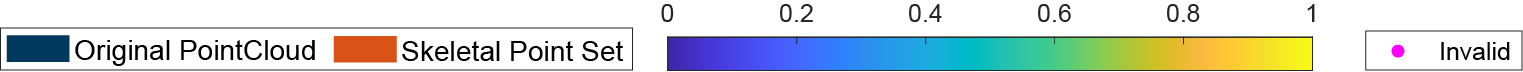}}
  \end{tabular}
  \caption{Evaluation of skeletonization results for normal, noise-induced, and sparse point cloud forms, assessed using the proposed metrics. The color indicates the local scoring along the skeleton.}
  \label{fig.graphic_table}
\end{figure*}
\begin{rem}
    The overall smoothness of the skeletal point set is defined as the average smoothness of its individual points. The overall smoothness of the curve skeleton is the edge-length weighted average smoothness of each vertex.
\end{rem}

As shown in Fig.~\ref{fig.skeleton_smoothness}, the smoothness for a point of the skeletal point set or vertices of the curve skeleton, given by Eq.~(\ref{eq.skeletal_point_smoothness}) and Eq.~(\ref{eq.curve_point_smoothness}), respectively, is able to show local variations in the tangent or curve direction. Those points/vertices where the direction changes sharply are assigned higher values of smoothness, while others are assigned smaller smoothness values. Although the overall smoothness of either the skeletal point set or curve skeleton (Definition~\ref{fig.skeleton_smoothness}) indicates the smoothness of the skeletal representations, the smoothness of the local points/vertices is of more importance. Local properties significantly influence motion planning strategies, such as a mobile robot following a trajectory~\cite{noel2023skeleton} or a grasping mechanism aligning with centralized local shape direction~\cite{vahrenkamp2018planning}. Poor local smoothness in curves can lead to unrealizable trajectories and introduce singularities, disrupting the continuity of planners or controllers.

\section{Results \& Discussion}
\label{resultsandiscussion}
This section aims to demonstrate a comprehensive evaluation and discussion on skeletonization results based on the proposed metrics in Section~\ref{EvaluationExperiments}, followed by a discussion of the desired properties in the application in robotics in Section~\ref{RoboticDiscussion}.

\subsection{Evaluation Experiments on Skeletonization Results}
\label{EvaluationExperiments}
We first evaluated skeletonization results obtained using the Laplacian-based contraction method~\cite{cao_point_2010}. Notably, our evaluation metric is adaptable to various skeletonization methods by simply importing the skeletal surface, represented by a point set or the curve, into the provided open-access toolbox. To generalize and expand the applicability umbrella of our proposed metrics, we introduced controlled degradations to the point cloud, such as adding noise, increasing sparsity, and perturbing normals. These variations in input are designed to assess the sensitivity of our evaluation method and its ability to distinguish differences in skeletonization quality, with degraded inputs expected to result in poorer skeletonization outcomes. The experiments were implemented in MATLAB, and all computations were performed on a machine with an i5-13500H CPU and 16GB of memory. The threshold values $c^*$, $\mathfrak{c}^*$ (Definition~\ref{def.overall_centeredness}), and $\beta^*$ (Proposition~\ref{prop.bounding_judgement}) were all set to 0.75. Additionally, the spatial size of the point cloud data was normalized, ensuring that the diagonal length of the bounding box $\epsilon_{\text{max}}$ introduced in Proposition~\ref{prop.topological_smilarity} was 1.6. Since centeredness and smoothness metric under our computation methods might not be available, the values for incomputable points are invalid and marked in magenta in Fig.~\ref{fig.graphic_table}.

We conducted our study on various inanimate objects and animal shapes using point cloud data collected from multiple datasets \cite{Wu_2023_CVPR_OmniObject3D,dobbs2024quantifying,PlanetaryPitsCavesDataset}. Figure~\ref{fig.graphic_table} and Table~\ref{tab.numeric_results} present selected examples, including both well-formed skeletons and those with structural issues. We analyze these cases, highlighting the differences and demonstrating how our metrics enable their detection. At first, the topology preservation of the skeletal shape is quantified using the topological distance score, derived from the barcode of persistent homology patterns. As shown in Fig.~\ref{fig.graphic_table}, poorer topological alignment results in higher topological distance scores, measured by bottleneck distance~(\ref{eq.bottleneck_dis}) and Wasserstein distance~(\ref{eq.wasserstein_dis}). The distance values in current settings are capped by the shape bounding box diagonal, 1.6. A lower distance value means higher topological similarity. The bottleneck distance captures significant topological changes, while the Wasserstein distance reflects the average change. The skeletal point set results of 4096 points input are considered good topology preservation if both distance values are below 0.02 ($d^*<0.02$ as introduced in Proposition~\ref{prop.topological_smilarity}). For example, the topology is well-preserved by the skeletal point set results for hammer input with 4096 points (first row of Fig.~\ref{fig.graphic_table}). Adding $5\%$ Gaussian noise to the hammer point cloud with 4096 points or using a sparser input with 1024 points approximately doubled both topological distances between the resulting skeletal point set and the input point cloud (Fig.~\ref{fig.graphic_table}, row 1-3). In the skeletal point set of the animal bear toy (last row), greater structural changes under the same input settings resulted in more than twice the topological distances compared to the hammer. In the biscuit shape (fifth row), significant shape changes in edge regions contribute to a very large bottleneck distance (0.0561), while the Wasserstein distance is relatively lower (0.0203) as the majority planar structure of the skeletal point set is aligned with the input shape. Since topological similarity is based on $H_0$ features (connected components), neighboring relationships influence the scores. Consequently, the difference in topological distance scores between sparse and dense point clouds (third vs. first row) may be larger than expected. Table~\ref {tab.numeric_results} confirms that sparse and noisy inputs degrade topology preservation, leading to higher topological distance scores.

\begin{table*}[t!]
    \centering
    \begin{tabular}{  m{3.4cm} m{2.3cm} m{3.1cm} m{2.6cm} m{2.6cm}  }
    \hline
         \textbf{Shapes} &  $d_B (\ref{eq.bottleneck_dis})/d_{W_p} (\ref{eq.wasserstein_dis})\downarrow$ & $\mathcal{B}_{P_s \circ P_o} (\ref{eq.boundedness_skeletal_points})/\mathcal{B}_{G_s \circ P_o} (\ref{eq.boundedness_curve})\uparrow$& $\mathcal{C}_s(P_s) /\mathcal{C}_s(G_s) (\ref{eq.overal_centeredness})\uparrow$ &$S(P_s)(\ref{eq.smoothness_skeletal_points})/\mathfrak{S}_{G_s} (\ref{eq.smoothness_curve})\uparrow $\\ \hline\hline
         {Dumbbell (4096 pts)\cite{Wu_2023_CVPR_OmniObject3D}} & 0.0218/0.0179  & 1.0/1.0       & 0.146/0.803    & 0.592/0.766 \\
         {Dumbbell (5\% noise)}& 0.0347/0.0239  & 0.994/1.0       & 0.060/0.295 & 0.559/0.914\\
         {Dumbbell (1024 pts)} & 0.0418/0.0379  & 0.967/0.939  & 0.021/0.364 & 0.601/0.974\\ \hline
         
         {Steamed Bun (4096 pts)} & 0.0309/0.0290  & 0.993/0.910       & 0.001/0.000    & 0.464/0.953 \\
         {Steamed Bun (5\% noise)}& 0.0439/0.0349  & 0.981/0.804       & 0.0/0.0 & 0.368/1.0\\
         {Steamed Bun (1024 pts)} & 0.0987/0.0340  & 0.979/0.983  & 0.008/0.023 & 0.347/0.854\\ \hline
         
       {Toy Plant (4096 pts)} & 0.0230/0.0214  & 0.993/0.965       & 0.277/0.358    & 0.583/0.755 \\
         {Toy Plant (5\% noise)}& 0.0350/0.0246  & 0.998/1.0       & 0.076/0.177 & 0.542/0.898\\
         {Toy Plant (1024 pts)} & 0.0535/0.0444  & 0.970/0.891  & 0.044/0.288 & 0.654/0.962\\ \hline
         
          {Skateboard (4096 pts)} & 0.0157/0.0144  & 1.0/1.0       & 0.223/0.636    & 0.734/0.918 \\
         {Skateboard (5\% noise)}& 0.0473/0.0173  & 0.999/0.995       & 0.185/0.509 & 0.687/0.953\\
         {Skateboard (1024 pts)} & 0.0247/0.0241  & 0.831/0.588  & 0.051/0.140 & 0.787/0.957\\ \hline

       {Banana (4096 pts)} & 0.0218/0.0154  & 1.0/1.0       & 0.788/0.969    & 0.704/0.828 \\
         {Banana (5\% noise)}& 0.1033/0.0188  & 0.992/0.855       & 0.157/0.136 & 0.557/0.934\\
         {Banana (1024 pts)} & 0.0321/0.0322  & 0.982/0.895  & 0.165/0.147 & 0.559/0.844\\ \hline

       {Knife (4096 pts)} & 0.0170/0.0112  & 0.986/0.980       & 0.255/0.629    & 0.705/0.880 \\
         {Knife (5\% noise)}& 0.0209/0.0125  & 0.968/0.980       & 0.248/0.311 & 0.680/0.846\\
         {Knife (1024 pts)} & 0.0212/0.0219  & 0.851/0.813  & 0.124/0.337 & 0.716/0.940\\ \hline

        {Synthetic tree (8936 pts)~\cite{dobbs2024quantifying}} & 0.0046/0.0048  & 0.830/0.721       & 0.284/0.102    & 0.904/0.922 \\
         {Synthetic tree (5\% noise)}& 0.0272/0.0158  & 0.996/1.0       & 0.233/0.163 & 0.628/0.867\\
         {Synthetic tree (1730 pts)} & 0.0077/0.0060  & 0.653/0.503  & 0.124/0.045 & 0.921/0.946\\ \hline
         
         {Cave (9769 pts)~\cite{PlanetaryPitsCavesDataset}} & 0.0297/0.0136  & 1.0/1.0       & 0.029/0.147    & 0.567/0.937 \\
         {Cave (5\% noise)}& 0.0745/0.0183  & 0.986/0.905       & 0.010/0.031 & 0.415/0.934\\
         {Cave (1730 pts)} & 0.261/0.0287  & 0.993/0.920  & 0.035/0.390 & 0.660/0.890
         \\ \hline\hline
    \end{tabular}
    \caption{Quantitative results of skeletonization evaluation. The input cave point cloud is sectioned and capped from the original data to ensure a closed shape for skeletonization. Different resolutions of cave and synthetic tree point clouds are acquired by grid-averaged sampling. }
    \label{tab.numeric_results}
\end{table*}

As discussed in the previous section, local performance is the most meaningful aspect of skeletonization evaluation. Fig.~\ref{fig.graphic_table} also visualizes the quality of local skeleton components using color intensity representation based on multiple metrics. In the hammer and scissor models, all skeletal components remain well-bounded, whereas, in the biscuit and animal bear models, some skeletal components extend beyond the surface boundary, indicated by the darker colors. Regarding sensitivity, boundedness responds to skeleton disturbances or deformations caused by sparse point clouds and reflects these changes effectively.  For centeredness, the hammerhead becomes less centered as the input point cloud becomes sparser, as shown by the cooler colors. Besides, as the noises are added to the input point cloud, the point cloud is not contracted properly, and some points in the resultant skeletal point set are less centered relative to the noisy input. The proposed metric also captures variations in centeredness across different regions, as demonstrated in the scissor model. The visual presentation shows that Laplacian-based skeletonization is highly sensitive to point cloud size, with sparse data being more damaging than noise, as noted in previous studies. However, the metrics consistently reflect these changes, confirming their sensitivity and aiding in quality assessment and convergence stability. According to the results,  one limitation of the centeredness metric is the sensitivity difference between the centeredness evaluation of the skeletal point set and the curve skeleton due to the centeredness approximation difference among those two quite different shape representations. The centeredness calculation of skeletal points relies on all neighboring points, while the centeredness of the curve skeleton is calculated based on only radial neighbor points. As shown in Fig.~\ref{fig.graphic_table}, such as the third row of the figure, the centeredness value of the skeletal point set is with more variation among skeletal components and tends to be lower than that of curve skeleton, especially around the ending point of the skeleton. There are two types of points with invalid values of centeredness. One is when the curve points coincident with the curve vertices, for which the centeredness is ambiguous. The other type is related to one limitation of our method, which is unable to give a valid centeredness value for where the points from the input of the corresponding skeletal components are very sparse, leading to ambiguous shape surface representation and the corresponding locus.

Similarly, for smoothness, the metric effectively captures local curvature change rates. Skeleton components with abrupt directional changes, such as the legs of the animal bear (last row), exhibit lower smoothness values, while smoother regions, like the body, have higher values. As shown in Table~\ref{tab.numeric_results} and Fig.~\ref{fig.graphic_table}, sparse and noisy input data generally lead to lower boundedness and centeredness scores, though minor variations exist. However, smoothness is more dependent on the intrinsic skeleton structure rather than input data quality, resulting in weaker correlations with data variations, as observed in both figures. The smooth metric does not give the value for curve skeleton endpoint vertices and the joint vertices, which is meaningless.

Regarding the overall score of those metrics, the skeletonization results, which perform well with respect to all those four metrics are recognized as good skeletons, while good skeletons might have detailed requirements in specific definitions. While our methods distinguish the skeletonization performance in shape representation, limitation exists, including the sensitivity in point cloud density for topological similarity, boundedness, and unavailability for local centeredness evaluation, where the points of input shape are extremely sparse, leading to ambiguous object surface representation and the corresponding locus of the shape.
 
\subsection{Desired Properties in Robotic Applications}
\label{RoboticDiscussion}
While the skeleton of a shape is always expected to be well-bounded, the significance of other skeletonization metrics may vary depending on different robotic application scenarios. Topological similarity is particularly crucial when using skeleton information for robotic manipulation~\cite{varava2016caging,vahrenkamp2018planning}. As shown in Fig.~\ref{fig.graphic_table}, the skeletonization results of objects such as the hammer and scissors visually follow the original topology and exhibit a small topological distance according to persistent homology patterns. Thus, these skeletonization results are of higher quality for robotic manipulation analysis. Notably, objects whose skeletonization results well preserve the topological structure using the Laplacian-based method are more likely to be cylindrical in shape.  

Centeredness also plays a key role in grasping planning when proposing local grasping candidates~\cite{vahrenkamp2018planning}. Taking the hammer as an example, the first row of Table~\ref{fig.graphic_table} shows that the skeleton’s centeredness is higher in the handle component and part of the hammerhead, where grasping confidence is higher. However, the joint and endpoints of the hammer skeleton exhibit lower centeredness, making them ambiguous for grasping analysis. A well-centered skeleton component is more likely to produce accurate grasping candidates.  

In agricultural manipulation, semantic structure, including both semantic labels and topological information, is of primary importance~\cite{you2022semantics}. According to the evaluation results for the synthetic tree shown in Table~\ref{tab.numeric_results}, noise significantly affects topology preservation, which is critical given that noise is common in agricultural sensing scenarios. This is expected, as tree branches are typically thin, making their topological structures more susceptible to damage from noise.  

For surface reconstruction, the work of Wu et al.~\cite{wu2020skeleton} may rely on skeletonization performance in terms of topological similarity to accurately traverse the corresponding point cloud data. In contrast, for navigation, centeredness is less critical, whereas smoothness and topological similarity are of greater importance~\cite{cornea2005applications}. As illustrated in Table~\ref{tab.numeric_results}, the resolution of the input point cloud significantly impacts topology preservation in skeletonization. Higher-resolution point clouds increase the computational load, whereas very low-resolution input degrades skeletonization performance, which makes these proposed metrics valuable for checking the quality of the achieved skeleton or median curve.

Smoothness is also essential for navigation, as the number of vertices in the curve skeleton affects the consistency of the navigation path, necessitating curve-smoothing operations for effective path planning. For real-time applications, including navigation and robotic manipulation, computation time is another critical factor affecting planning confidence. In our experiments, the average evaluation time for skeletonization of a point cloud shape with 4096 points was 34 seconds. The most computationally intensive parts were topological similarity and boundedness evaluation, accounting for approximately 76\% and 23\% of the computation time, respectively.

To summarize, the proposed evaluation metric can measure the performance of both the skeletal point set and curve skeleton from multiple perspectives. However, the sensitivity to point density changes of input point cloud of the proposed boundedness and topological similarity metrics, and the limitations for centeredness evaluation are also shown in the results. Besides, the importance of those metrics varies in different application scenarios which means the performance of skeletonization methods in different scenarios might be different. Overall, this metric will open a new area of research on understanding the quality of skeletons for better performance in robotic applications from grasping to navigation.




\section{Conclusion}
In this work, we propose a novel formal geometric evaluation metric for skeletonization assessment, drawing inspiration from previous studies on certain metrics. Moving beyond traditional subjective visual judgment, we introduce numerical and representative tools to assess skeletonization performance from a geometric perspective. Our evaluation, conducted on real-scanned point cloud skeletonization across different resolutions and noise levels, demonstrates that the proposed method effectively captures skeletonization performance in various aspects, including topological structure, boundedness, centeredness, and smoothness. Additionally, we assess skeletonization for point clouds from different robotic application scenarios and discuss the relative importance of multiple metrics.  

This study represents the first numerical skeletonization evaluation with a discussion of application scenarios. However, limitations remain, including sensitivity to point cloud resolution. Regarding different robotic application scenarios, we provide a general evaluation framework. In the future, we plan to improve the computation time of the metrics by focusing on a specific robotic manipulation application, such as the rolling contact problem.


%



\ifCLASSOPTIONcompsoc
  \section*{Acknowledgments}
\else
  \section*{Acknowledgment}
\fi
This work was partially supported by the Royal Society research grant under Grant \text{RGS\textbackslash R2\textbackslash 242234}. Also, this work was partially supported by China Scholarship Council via a stipend (No. 202006760092).

\ifCLASSOPTIONcaptionsoff
  \newpage
\fi



%
\bibliographystyle{IEEEtran}
\bibliography{skeleton_ref} 

%

\begin{IEEEbiography}[{\includegraphics[width=1in,height=1.25in,clip,keepaspectratio]{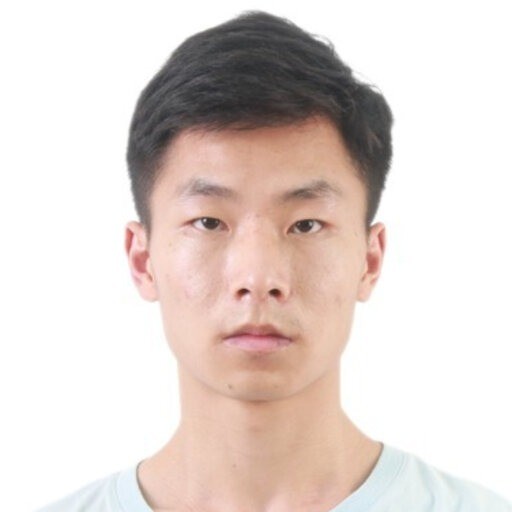}}]{Qingmeng Wen} (Student Member, IEEE) received the B.S. degree in automation engineering from Huazhong Agricultural University, Wuhan, China, in 2021. He is currently pursuing a Ph.D. degree with Cardiff University, U.K., His research interests include robotics, geometry processing, and computer vision.
\end{IEEEbiography}

\begin{IEEEbiography}
[{\includegraphics[width=1in,height=1.25in,clip,keepaspectratio]{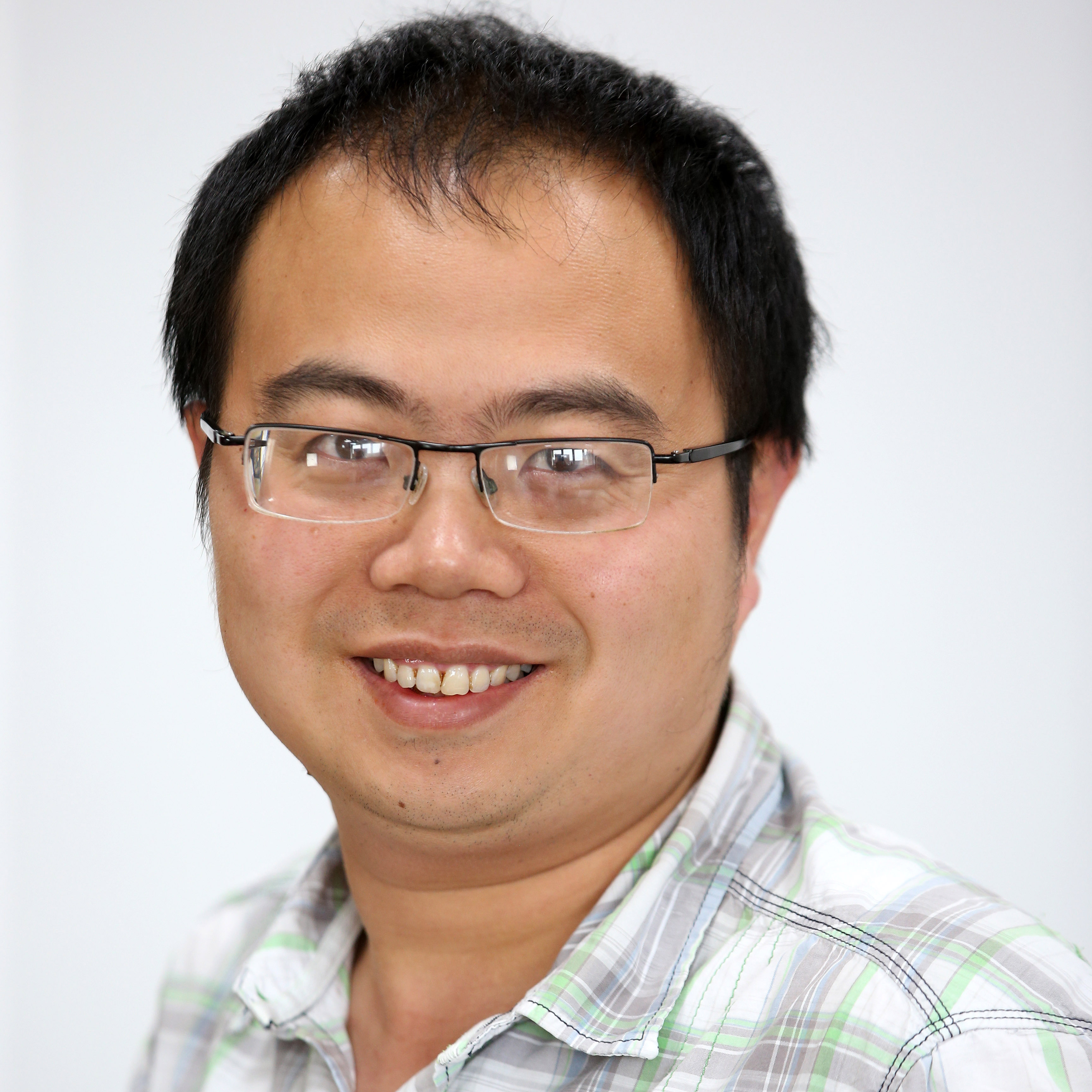}}]
{Yu-Kun Lai} (Senior Member, IEEE) received the bachelor’s and Ph.D. degrees in computer science from Tsinghua University in 2003 and 2008, respectively. He is currently a Professor with the School of Computer Science and Informatics, Cardiff University. His research interests include computer graphics, geometry processing, image processing, and computer vision. He is on the Editorial Boards of \emph{IEEE Transactions on Visualization and Computer Graphics} and \emph{The Visual Computer}.
\end{IEEEbiography}

\begin{IEEEbiography}[{\includegraphics[width=1in,height=1.25in,clip,keepaspectratio]{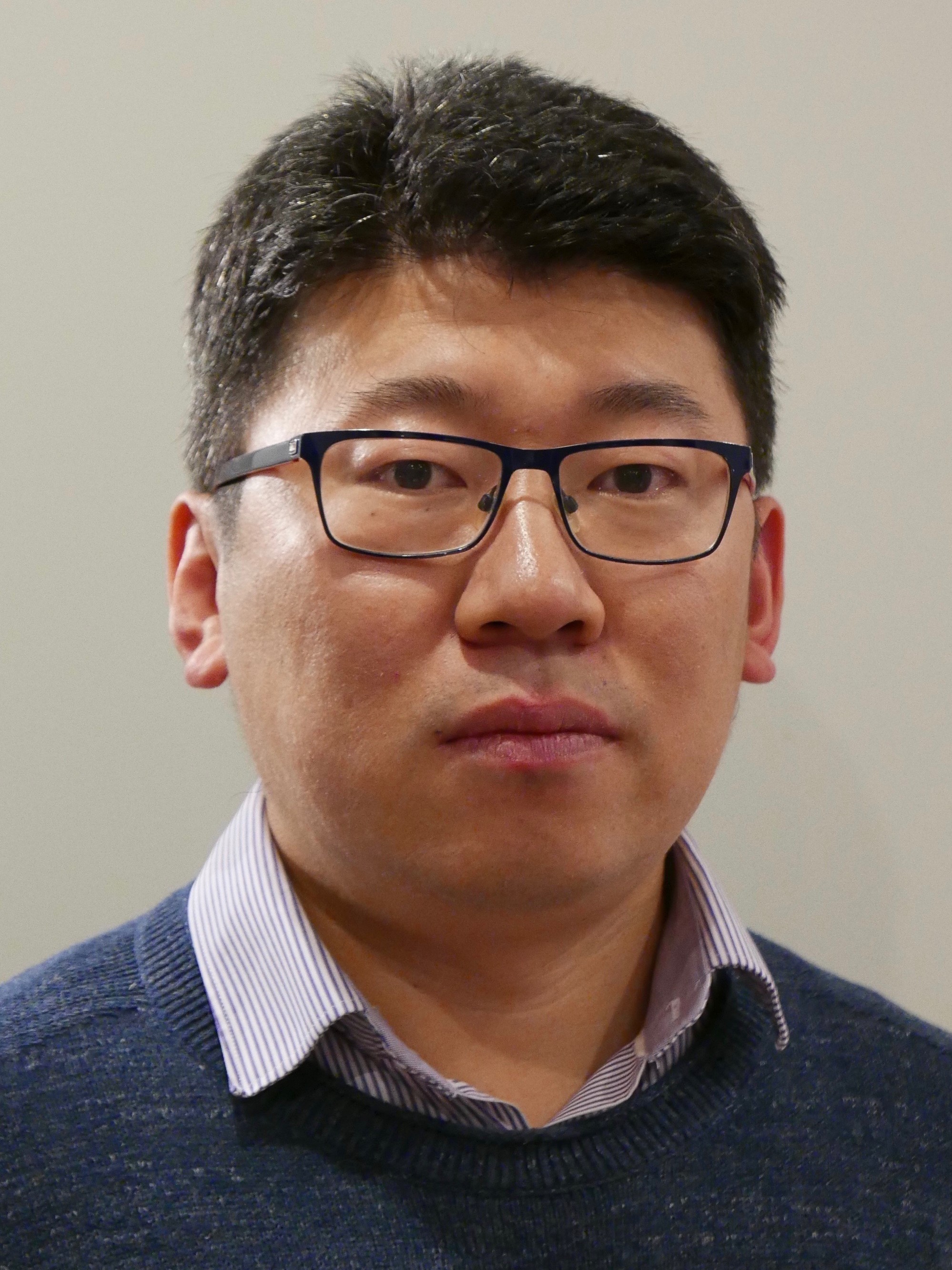}}]{Ze Ji} (Member, IEEE) received the B.Eng. degree from Jilin University, Changchun, China, in 2001, the M.Sc. degree from the University of Birmingham, Birmingham, U.K., in 2003, and the Ph.D. degree from Cardiff University, Cardiff, U.K., in 2007., He is a Reader with the School of Engineering, Cardiff University, U.K. Prior to his current position, he was working in industry (Dyson, Lenovo, etc) on autonomous robotics. His research interests are cross-disciplinary, including autonomous robot navigation, robot manipulation, robot learning, computer vision, simultaneous localization and mapping (SLAM), acoustic localization, and tactile sensing.
\end{IEEEbiography}

\begin{IEEEbiography}[{\includegraphics[width=1in,height=1.25in,clip,keepaspectratio]{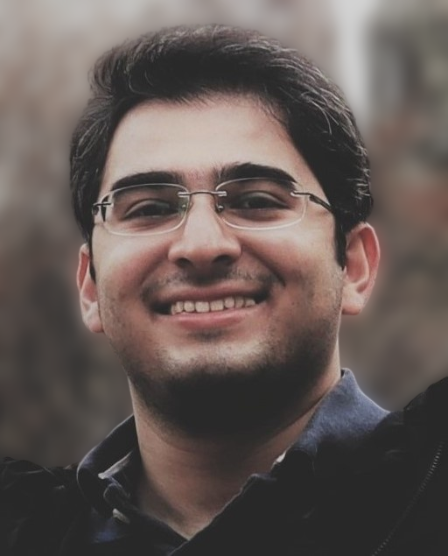}}]{Seyed Amir Tafrishi} (Member, IEEE) received his M.Sc. degree in control systems engineering from the University of Sheffield in 2014, UK, and Ph.D. degree in mechanical engineering from Kyushu University, Japan in 2021. 
	
Dr. Tafrishi is currently a Lecturer at Engineering School, Cardiff University, UK. He is the director of the Geometric Mechanics and Mechatronics in Robotics (gm$^2$R) lab. He was a Specially Appointed Assistant Professor working on Mooonshot R \& D project JST at Tohoku University, Japan, between 2021-2022. Since 2014, he has been a visiting researcher at  University of Sheffield, the Mechatronics Lab at METU, Turkey, and the Fluid Mechanics Lab at the University of Tabriz, Iran. His research interests include robotics, mechanism design, geometric mechanics, under-actuated systems.
\end{IEEEbiography}
\vfill






\end{document}